\documentclass{article}

\usepackage[final,nonatbib]{neurips_2021}

\usepackage[utf8]{inputenc} %
\usepackage[T1]{fontenc}    %
\usepackage[colorlinks = true, urlcolor  = blue, citecolor = black]{hyperref}       %
\usepackage{url}            %
\usepackage{booktabs}       %
\usepackage{amsfonts}       %
\usepackage{nicefrac}       %
\usepackage{microtype}      %

\usepackage[utf8]{inputenc}
\usepackage{graphicx}
\usepackage{color}
\usepackage{times}
\usepackage{amsmath,mathtools}
\usepackage{amsthm}	
\usepackage{amssymb}
\usepackage{enumerate}
\usepackage{subfigure}
\usepackage[algo2e,inoutnumbered,linesnumbered,algoruled,vlined]{algorithm2e}
\usepackage{algpseudocode}
\usepackage{booktabs}
\usepackage{tabularx}
\usepackage{array}
\usepackage{bm}
\usepackage{enumitem}
\usepackage{xcolor}
\usepackage{wrapfig}
\usepackage{tabularx}
\usepackage{cutwin}

\usepackage{cleveref}

\usepackage{dsfont}

\usepackage{amsmath}
\usepackage{amssymb}%
\usepackage{booktabs}
\usepackage{capt-of} %
\usepackage{cuted} %
\usepackage{etoolbox} %
\usepackage{float}
\usepackage{fontawesome}
\usepackage{graphicx}
\usepackage{makecell}
\usepackage{marvosym}
\usepackage{multirow}
\usepackage{pifont}%
\usepackage{tabularx}
\usepackage{wrapfig}

\newcommand{\filluptopage}[1]{%
  \clearpage
  \loop\ifnum\value{page}<#1\relax
    \null\clearpage
  \repeat
  \loop\ifnum\value{page}=#1\relax
    \null\clearpage
  \repeat
}

\makeatletter
\def\blfootnote{\xdef\@thefnmark{}\@footnotetext}
\makeatother

\usepackage{amsmath,amsfonts,bm}

\newcommand{\veps}{\varepsilon}

\def\1{\bm{1}}

\DeclareMathAlphabet{\mathsfit}{\encodingdefault}{\sfdefault}{m}{sl}
\SetMathAlphabet{\mathsfit}{bold}{\encodingdefault}{\sfdefault}{bx}{n}

\newcommand{\E}{\mathbb{E}}

\newcommand{\R}{\mathbb{R}}

\DeclareMathOperator*{\argmin}{arg\,min}

\newcommand{\MST}{\mathrm{MST}}
\newcommand{\BOX}{\mathrm{Box}}

\newcommand{\Graph}{\mathcal{G}}

\newcommand{\PH}{\mathrm{PH}}
\newcommand{\dimPH}{\mathrm{dim}_\PH}
\newcommand{\dimMST}{\mathrm{dim}_\MST}
\newcommand{\dimBox}{\mathrm{dim}_\BOX}

\newcommand{\chainComplex}{\mathcal{C}}
\newcommand{\compX}{\mathcal{W}}
\newcommand{\cech}{\text{\v{C}ech}_r}
\newcommand{\VR}{\text{VR}}
\newcommand{\rips}{\text{VR}_r}
\newcommand{\birth}{\mathrm{birth}}
\newcommand{\death}{\mathrm{death}}
\newcommand{\PHrobust}{\PH_0\mathrm{RANSAC}}

\newcommand{\Q}{\mathbf{Q}}

\newtheorem{thm}{Theorem}

\newtheorem{prop}{Proposition}
\newtheorem{dfn}{Definition}

\newcommand{\etal}{\textit{et al}. }
\newcommand{\ie}{\textit{i}.\textit{e}.}
\newcommand{\eg}{\textit{e}.\textit{g}.}

\usepackage{cleveref}

\newtheorem{assumption}{\textbf{H}\hspace{-3pt}}
\Crefname{assumption}{\textbf{H}\hspace{-3pt}}{\textbf{H}\hspace{-3pt}}
\crefname{algorithm}{\text{Alg.}}{\text{Alg.}}
\crefname{assumption}{\textbf{H}}{\textbf{H}}
\crefname{equation}{\text{Eq}}{\text{Eq}}
\crefname{definition}{\text{Dfn.}}{\text{Dfn.}}
\crefname{proposition}{\text{Prop.}}{\text{Prop.}}
\crefname{lemma}{\text{Lemma}}{\text{Lemma}}
\crefname{dfn}{\text{Dfn.}}{\text{Dfn.}}
\crefname{thm}{\text{Thm.}}{\text{Thm.}}
\crefname{tab}{\text{Tab.}}{\text{Tab.}}
\crefname{fig}{\text{Fig.}}{\text{Fig.}}
\crefname{table}{\text{Tab.}}{\text{Tab.}}
\crefname{figure}{\text{Fig.}}{\text{Fig.}}
\crefname{section}{\text{Sec.}}{\text{Sec.}}

\newcommand{\insertimageC}[5]{ %
\begin{figure}[#5]
\centering
\includegraphics[width=#1\linewidth, clip=true]{content/figures/#2}
\vspace{-1.5em}
\caption{#3}
\label{#4}
\end{figure}
}

\newcommand{\Alg}{\mathcal{A}}

\newcommand{\wb}{w}

\newcommand{\rset}{\mathbb{R}}
\newcommand{\pr}{\mathbb{P}}

\newcommand{\Dcal}{\mathcal{D}}

\newcommand{\Rcal}{\mathcal{R}}
\newcommand{\Wcal}{\mathcal{W}}
\newcommand{\Xcal}{\mathcal{X}}
\newcommand{\Ycal}{\mathcal{Y}}
\newcommand{\Zcal}{\mathcal{Z}}

\newcommand{\Mut}{M}

\Crefname{assumption}{\textnormal{\textbf{H}}\hspace{-3pt}}{\textnormal{\textbf{H}}\hspace{-3pt}}
\crefname{assumption}{\textnormal{\textbf{H}}}{\textnormal{\textbf{H}}}

\title{Intrinsic Dimension, Persistent Homology and Generalization in Neural Networks}

\author{Tolga Birdal\\
  {Stanford University}\\
  \texttt{tbirdal@stanford.edu}\\
  \And
  Aaron Lou\\
  {Stanford University}\\
  \texttt{aaronlou@stanford.edu}\\
  \And
  Leonidas Guibas\\
  {Stanford University}\\
  \texttt{guibas@cs.stanford.edu}\\
  \And
  Umut \c{S}im\c{s}ekli\\
  {INRIA \& ENS -- PSL Research University}\\
  \texttt{umut.simsekli@inria.fr}
}

\begin{document}

\maketitle

\begin{abstract}

Disobeying the classical wisdom of statistical learning theory, modern deep neural networks generalize well even though they typically contain millions of parameters. Recently, it has been shown that the trajectories of iterative optimization algorithms can possess \emph{fractal structures}, and their generalization error can be formally linked to the complexity of such fractals. This complexity is measured by the fractal's \emph{intrinsic dimension}, a quantity usually much smaller than the number of parameters in the network. Even though this perspective provides an explanation for why overparametrized networks would not overfit, computing the intrinsic dimension (\eg, for monitoring generalization during training) is a notoriously difficult task,  where existing methods typically fail even in moderate ambient dimensions. In this study, we consider this problem from the lens of topological data analysis (TDA) and develop a generic computational tool that is built on rigorous mathematical foundations. By making a novel connection between learning theory and TDA, we first illustrate that the generalization error can be equivalently bounded in terms of a notion called the 'persistent homology dimension' (PHD), where, compared with prior work, our approach does not require any additional geometrical or statistical assumptions on the training dynamics. Then, by utilizing recently established theoretical results and TDA tools, we develop an efficient algorithm to estimate PHD in the scale of modern deep neural networks and further provide visualization tools to help understand generalization in deep learning. Our experiments show that the proposed approach can efficiently compute a network's intrinsic dimension in a variety of settings, which is predictive of the generalization error. 

\end{abstract}

\section{Introduction}
\label{sec:intro}
In recent years, deep neural networks (DNNs) have become the de facto machine learning tool and have revolutionized a variety of fields such as natural language processing~\cite{devlin2018bert}, image perception~\cite{krizhevsky2012imagenet,rempe2021humor}, geometry processing~\cite{qi2017pointnet,zhao2020quaternion} and 3D vision~\cite{deng2018ppfnet,gojcic2020weakly}. Despite their widespread use, little is known about their theoretical properties. Even now the top-performing DNNs are designed by trial-and-error, a pesky, burdensome process for the average practitioner~\cite{elsken2019neural}. Furthermore, even if a top-performing architecture is found, it is difficult to provide performance guarantees on a large class of real-world datasets.

This lack of theoretical understanding has motivated a plethora of work focusing on explaining what, how, and why a neural network learns. To answer many of these questions, one naturally examines the generalization error, a measure quantifying the differing performance on train and test data since this provides significant insights into whether the network is learning or simply memorizing~\cite{zhang2021understanding}. However, generalization in neural networks is particularly confusing as it refutes the classical proposals of statistical learning theory such as uniform bounds based on the Rademacher complexity~\cite{bartlett2002rademacher} and the Vapnik–Chervonenkis (VC) dimension~\cite{vapnik1968uniform}. 

Instead, recent analyses have started focusing on the dynamics of deep neural networks. \cite{neyshabur2017,blier2018description,gao2016degrees} provide analyses on the final trained network, but these miss out on critical training patterns. To remedy this, a recent study 
\cite{simsekli2020hausdorff} connected generalization and the heavy tailed behavior of \emph{network trajectories}--a phenomenon which had already been observed in practice~\cite{simsekli2019tail,csimcsekli2019heavy,simsekli2020fractional,gurbuzbalaban2021heavy,camuto2021asymmetric,hodgkinson2020multiplicative,martin2019traditional}. \cite{simsekli2020hausdorff} further showed that the generalization error can be linked to the \emph{fractal dimension} of a parametric hypothesis class (which can then be taken as the optimization trajectories). Hence, the fractal dimension acts as a `capacity metric' for generalization.

While \cite{simsekli2020hausdorff} brought a new perspective to generalization, several shortcomings prevent application in everyday training. In particular, their construction requires several conditions which may be infeasible in practice: (i) topological regularity conditions on the hypothesis class for fast computation, (ii) a Feller process assumption on the training algorithm trajectory, and that (iii) the Feller process exhibits a specific diffusive behavior near a minimum. Furthermore, the capacity metrics in \cite{simsekli2020hausdorff} are not optimization friendly and therefore can't be incorporated into training.

In this work, we address these shortcomings by exploiting the recently developed connections between fractal dimension and topological data analysis (TDA). First, by relating the \emph{box dimension}~\cite{schroeder2009fractals} and the recently proposed \emph{persistent homology (PH) dimension}~\cite{schweinhart2020fractal}, we relax the assumptions in \cite{simsekli2020hausdorff} to develop a topological intrinsic dimension (ID) estimator. Then, using this estimator we develop a general tool for \emph{computing} and \emph{visualizing} generalization properties in deep learning. Finally, by leveraging recently developed differentiable TDA tools~\cite{Hofer17a,Hofer19a}, we employ our ID estimator to regularize training towards solutions that generalize better, even without having access to the test dataset.

Our experiments demonstrate that this new measure of intrinsic dimension correlates highly with generalization error, regardless of the choice of optimizer. Furthermore, as a proof of concept, we illustrate that our topological regularizer is able to improve the test accuracy and lower the generalization error. In particular, this improvement is most pronounced when the learning rate/batch size normally results in a poorer test accuracy.

Overall, our contributions are summarized as follows:
\begin{itemize}[itemsep=0.25pt,topsep=0pt,leftmargin=*]
    \item We make a novel connection between statistical learning theory and TDA in order to develop a generic computational framework for the generalization error.
    We remove the topological regularity condition and the decomposable Feller assumption on training trajectories, which were required in \cite{simsekli2020hausdorff}. This leads to a more generic capacity metric.
    \item Using insights from our above methodology, we leverage the differentiable properties of persistent homology to regularize neural network training.  Our findings also provide the first steps towards theoretically justifying recent topological regularization methods~\cite{bruel2019topology,chen2019topological}.
    \item We provide extensive experiments to illustrate the theory, strength, and flexibility of our framework.
\end{itemize}
We believe that the novel connections and the developed framework will open new theoretical and computational directions in the theory of deep learning. To foster further developments at the the intersection of persistent homology and statistical learning theory, we release our source code under: \href{https://github.com/tolgabirdal/PHDimGeneralization}{https://github.com/tolgabirdal/PHDimGeneralization}.

\vspace{-3mm}
\section{Related Work}
\label{sec:related}

\paragraph{Intrinsic dimension in deep networks}
Even though a large number of parameters are required to train deep networks~\cite{frankle2018lottery}, modern interpretations of deep networks avoid correlating model over-fitting or generalization to parameter counting. Instead, contemporary studies measure model complexity through the degrees of freedom of the parameter space~\cite{janson2015effective,gao2016degrees}, compressibility (pruning)~\cite{blier2018description} or intrinsic dimension~\cite{ansuini2019intrinsic,li2018measuring,ma2018dimensionality}. 
Tightly related to the ID, Janson~\etal~\cite{janson2015effective} investigated the \emph{degrees of freedom}~\cite{ghrist2010configuration} in deep networks and expected difference between test error and training error. 
Finally, LDMNet~\cite{zhu2018ldmnet} explicitly penalizes the ID regularizing the network training. 

\paragraph{Generalization bounds}

Several studies have provided theoretical justification to the observations that trained neural networks live in a lower-dimensional space, and this is related to the generalization performance. In particular, compression-based generalization bounds \cite{arora2018stronger,suzuki2018spectral,Suzuki2020Compression,hsu2021generalization,barsbey2021heavy} have shown that the generalization error of a neural network can be much lower if it can be accurately represented in lower dimensional space. Approaching the problem from a geometric viewpoint, \cite{simsekli2020hausdorff} showed that the generalization error can be formally linked to the fractal dimension of a parametric hypothesis class. This dimension indeed the plays role of the intrinsic dimension, which can be much smaller than the ambient dimension. When the hypothesis class is chosen as the trajectories of the training algorithm, \cite{simsekli2020hausdorff} further showed that the error can be linked to the heavy-tail behavior of the trajectories.

\paragraph{Deep networks \& topology} Previous works have linked neural network training and topological invariants, although all analyze the final trained network \cite{fernandez2021determining}. For example, in \cite{Rieck2019NeuralPA}, the authors construct Neural Persistence, a measure on neural network layer weights. They furthermore show that Neural Persistence reflects many of the properties of convergence and can classify weights based on whether they overfit, underfit, or exactly fit the data. In a parallel line of work, \cite{Dollr2019WhatDI} analyze neural network training by calculating topological properties of the underlying graph structure. This is expanded upon in \cite{Corneanu2020ComputingTT}, where the authors compute correlations between neural network weights and show that the homology is linked with the generalization error.

However, these previous constructions have been done mostly in an adhoc manner. As a result, many of the results are mostly empirical and work must still be done to show that these methods hold theoretically. Our proposed method, by contrast, is theoretically well-motivated and uses tools from statistical persistent homology theory to formally links the generalization error with the network training trajectory topology.

We also would like to note that prior work has incorporated topological loss functions to help normalize training. In particular, \cite{bruel2019topology} constructed a topological normalization term for GANs to help maintain the geometry of the generated 3d point clouds.

\section{Preliminaries Technical Background}
\label{sec:bg}
We imagine a point cloud $W=\{w_i\in\R^{d}\}$ as a geometric realization of a $d$-dimensional topological space $W \subset \compX \subset \rset^d$. $B_\delta(x) \subset \rset^d$ denotes the closed ball centered around $x \in \rset^d$ with radius $\delta$.

\begin{wrapfigure}[15]{r}{0.25\textwidth}
    \vspace{-4.75mm}
\includegraphics[width=0.25\textwidth]{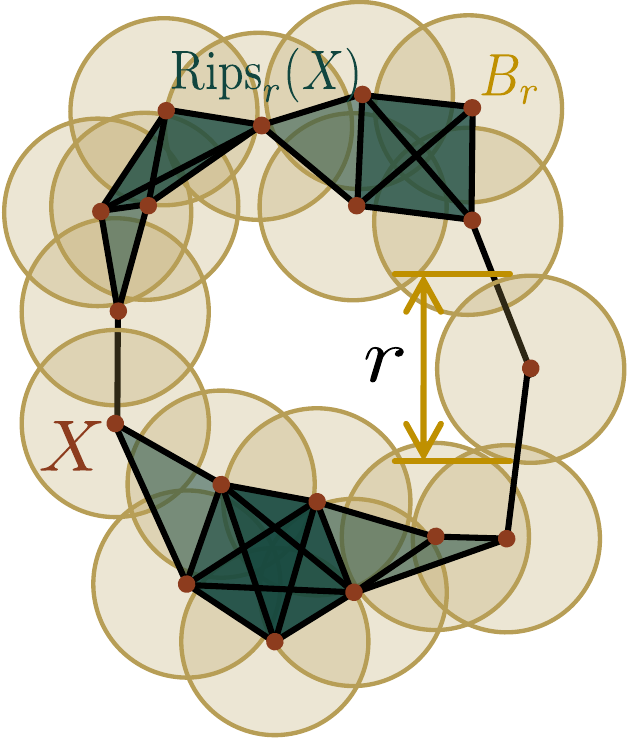}\vspace{-5mm}
    \caption{\small A visualization of a Vietoris-Rips complex computed using persistent homology (PH).}\vspace{-3mm}
    \label{fig:PH}
    \vspace{-3mm}
\end{wrapfigure}
\paragraph{Persistent Homology}
From a topological perspective, $\compX$ can be viewed a \emph{cell complex} composed of the disjoint union of $k$-dimensional balls or \emph{cells} $\sigma\in\compX$ \emph{glued} together. For $k=0,1,2,\dots$, we form a $\emph{chain complex}$ $\chainComplex(\compX) = \dots C_{k+1}(\compX)\xrightarrow[]{\partial_{k+1}} C_k(\compX) \xrightarrow[]{\partial_{k}} \dots$ by sequencing \emph{chain groups} $C_k(\compX)$, whose elements are equivalence classes of cycles, via boundary maps $\partial_k : C_k(\compX)\mapsto C_{k-1}(\compX)$ with $\partial_{k-1}\circ\partial_k \equiv 0$. In this paper, we work with finite \emph{simplicial complexes} restricting the cells to be simplices. 

The $k^{\mathrm{th}}$ homology group or \emph{$k$-dimensional homology} is then defined as the equivalence classes of $k$-dimensional cycles who differ only by a boundary, or in other words, the quotient group $H_k(\compX)=Z_k(\compX) / Y_k(\compX)$ where $Z_k(\compX)=\mathrm{ker}\,\partial_k$ and $Y_k(\compX)=\mathrm{im}\,\partial_{k+1}$. The generators or \emph{basis} of $H_0(\compX), H_1(\compX) \text{ and } H_2(\compX)$ describe the shape of the topological space $\compX$ by its connected components, holes and cavities, respectively. Their ranks are related to the \emph{Betti numbers} \ie $\,\beta_k=\mathrm{rank}(H_k)$.

\begin{dfn}[\v{C}ech and Vietoris-Rips Complexes]
    For $W$ a set of fine points in a metric space, the \v{C}ech cell complex $\cech(W)$ is constructed using the intersection of $r$-balls around $W$, $B_r(W)$: $\cech(W)=\big\{ Q\subset W : \cap_{x\in\Q} B_r(x)\neq 0 \big\}$. The construction of such complex is intricate. Instead, the Vietoris-Rips complex $\rips(W)$ closely approximates $\cech(W)$ using only the \emph{pairwise distances} or the intersection of two $r$-balls~\cite{reani2021cycle}: $\compX_r=\rips(W)=\big\{ Q\subset W : \forall x,x^\prime \in Q, \, B_r(x)\cap B_r(x^\prime)\neq 0 \big\}$. 
\end{dfn}
\begin{dfn}[Persistent Homology]
PH indicates a multi-scale version of homology applied over a \emph{filtration} $\{\compX_t\}_t:= \VR(W)\,:\,\forall (s\leq t)\,\compX_s \subset \compX_t \subset \compX$, keeping track of holes created (born) or filled (died) as $t$ increases. Each \emph{persistence module} $\PH_k(\VR(W))=\{\gamma_i\}_i$ keeps track of a single $k$-\emph{persistence cycle} $\gamma_i$ from $\birth$ to $\death$. We denote the entire lifetime of cycle $\gamma$ as $I(\gamma)$ and its length as $|I(\gamma)|=\death(\gamma)-\birth(\gamma)$. We will also use \emph{persistence diagrams}, 2D plots of all persistence lifetimes (death vs.\ birth). Note that for $\PH_0$, the \text{\v{C}ech} and VR complexes are equivalent.
\end{dfn}

Lifetime intervals are instrumental in TDA as they allow for extraction of topological features or summaries. Note that, each birth-death pair can be mapped to the cells that respectively created and destroyed the homology class,
defining a unique map for a persistence diagram, which lends itself to \emph{differentibility}~\cite{bruel2019topology,Hofer19a,Hofer17a}. We conclude this brief section by referring the interested reader to the well established literature of persistent homology~\cite{carlsson2014topological,edelsbrunner2010computational} for a thorough understanding.

\paragraph{Intrinsic Dimension} The intrinsic dimension of a space can be measured by using various  notions. In this study, we will consider two notions of dimension, namely the upper-box dimension (also called the Minkowski dimension) and the persistent homology dimension. The box dimension is based on covering numbers and can be linked to generalization via \cite{simsekli2020hausdorff}, whereas the PH dimension is based on the notions defined earlier in this section.

We start by the box dimension.
\begin{dfn}[Upper-Box Dimension]
    \label{def:mink}
    For a \textbf{bounded metric space} $\Wcal$, let $N_\delta(\Wcal)$ denote the maximal number of disjoint closed $\delta$-balls with centers in $\Wcal$. The \emph{upper box dimension} is defined as:
    \begin{equation}
        \dimBox \Wcal = \limsup_{\delta\to 0} \Big({\log(N_\delta(\Wcal))}/{\log(1/\delta)}\Big).
    \end{equation}
\end{dfn}

We proceed with the PH dimension. First let us define an intermediate construct, which will play a key role in our computational tools.
\begin{dfn}[$\alpha$-Weighted Lifetime Sum]
    For a \textbf{finite set} $W \subset \Wcal \subset \rset^d$,
    the weighted $i^\mathrm{th}$ homology lifetime sum is defined as follows:
    \begin{align}
        E^i_\alpha(W) = \sum\limits_{\gamma\in\PH_i( \mathrm{VR}(W) )} |I(\gamma)|^\alpha,
    \end{align}
    where $\PH_i(\mathrm{VR}(W))$ is the $i$-dimensional persistent homology of the \v{C}ech complex on a finite point set $W$ contained in $\Wcal$ and $|I(\gamma)|$ is the persistence lifetime as explained above.
\end{dfn}

Now, we are ready to define the PH dimension, which is the key notion in this paper.
\begin{dfn}[Persistent Homology Dimension]
\label{def:phdim}
The $\PH_i$-dimension of a \textbf{bounded metric space} $\Wcal$ is defined as follows:
\begin{equation}
    \dimPH^i \Wcal := \inf \big\{ \alpha \,:\, E^i_\alpha(W) < C; \quad \exists C>0, \forall \text{ finite } W\subset \Wcal\big\}.
\end{equation}
\end{dfn}
In words, $\dimPH^i \Wcal$ is the smallest exponent $\alpha$ for which $E_\alpha^i$ is uniformly bounded for all finite subsets of $\Wcal$.

\section{Generalization Error via Persistent Homology Dimension}
\label{sec:method}

In this section, we will illustrate that the generalization error can be linked to the $\PH_0$ dimension. Our approach is based on the following fundamental result. 
\begin{thm}[\cite{kozma2006minimal,schweinhart2019persistent}]
\label{thm:kozma}
Let $\Wcal \subset \rset^d$ be a bounded set. Then, it holds that:
$$\dimPH \Wcal := \dimPH^0 \Wcal = \dimBox \Wcal.$$
\end{thm}
In the light of this theorem, we combine the recent result showing that the generalization error can be linked to the box dimension \cite{simsekli2020hausdorff}, and Theorem~\ref{thm:kozma}, which  shows that, for bounded subsets of $\rset^d$, the box dimension and the PH dimension of order $0$ agree.

By following the notation of \cite{simsekli2020hausdorff}, we consider a standard supervised learning setting, where the data space is denoted by $\Zcal =  \Xcal \times \Ycal$, and $\Xcal$ and $\Ycal$ respectively denote the features and the labels. We assume that the data is generated via an unknown data distribution $\Dcal$ and we have access to a training set of $n$ points, i.e., $S = \{z_1, \dots, z_n \}$, with the samples $\{z_i\}_{i=1}^n$ are independent and identically (i.i.d.) drawn from $\Dcal$.

We further consider a parametric hypothesis class $\Wcal \subset \rset^d$, that potentially depends on $S$. We choose $\Wcal$ to be \emph{optimization trajectories} given by a training algorithm $\Alg$, 
which returns the entire (random) trajectory of the network weights in the time frame $[0,T]$, such that $[\Alg(S)]_t =w_t$ being the network weights returned by $\Alg$ at `time' $t$, and $t$ is a continuous iteration index. Then, in the set $\Wcal$, we collect all the network weights that appear in the optimization trajectory: $$\Wcal := \{\wb \in \rset^d : \exists t \in [0,T],  \wb = [\Alg(S)]_t \}$$
where we will set $T=1$, without loss of generality.

To measure the quality of a parameter vector $w \in \Wcal$, we use a loss function $\ell : \rset^d \times \Zcal \mapsto \rset_+$, such that $\ell(\wb, z)$ denotes the loss corresponding to a single data point $z$. We then denote the population and empirical risks respectively by $\Rcal(\wb) := \E_z [\ell(\wb,z)] $ and $\hat{\Rcal}(\wb,S) := \frac1{n} \sum\nolimits_{i=1}^n \ell(\wb, z_i)$. The generalization error is hence defined as $|\hat{\Rcal}(\wb,S)- \Rcal(w)|$.

We now recall \cite[Asssumption H4]{simsekli2020hausdorff}, which is a form of algorithmic stability \cite{bousquet2002stability}. Let us first introduce the required notation.
For any $\delta>0$, consider the fixed grid on $\rset^d$,
$$G = \left\{\left( \frac{(2j_1+1) \delta}{2\sqrt{d}},\dots,\frac{(2j_d+1) \delta}{2\sqrt{d}} \right): j_i \in \mathbb{Z}, i=1,\dots,d\right\},$$
and define the set $N_\delta := \{x\in G: B_\delta(x)\cap \Wcal \neq \emptyset\}$, that is the collection of the centers of each ball that intersect $\Wcal$.
\begin{assumption}\label{asmp:decoupling2}
  Let $\Zcal^\infty := (\Zcal \times \Zcal \times \cdots)$ denote the countable product endowed with the product topology and let $\mathfrak{B}$ be the Borel $\sigma$-algebra generated by $\Zcal^\infty$.
Let $\mathfrak{F}, \mathfrak{G}$ be the sub-$\sigma$-algebras of $\mathfrak{B}$ generated by the collections of random variables given by
$\{ \hat{\Rcal}(w,S): w \in \rset^d, n \geq 1\}$ and
$\Big\{ \mathds{1}\left\{w\in N_{\delta}\right\}: \delta\in \mathbb{Q}_{>0}, w\in G, n \geq 1 \Big\}$ respectively.
There exists a constant $\Mut \geq 1$ such that for any $A\in \mathfrak{F}$, $B\in \mathfrak{G}$ we have
$\pr\left[ A \cap B\right] \leq \Mut \pr\left[ A \right] \pr[B].$
\end{assumption}

The next result forms our main observation, which will lead to our methodological developments.
\begin{prop}
\label{thm:dimm_app2}
Let $\Wcal \subset \rset^d$ be a (random) compact set. Assume that \Cref{asmp:decoupling2} holds, $\ell$ is bounded by $B$ and $L$-Lipschitz continuous in $\wb$. Then, for $n$
sufficiently large, we have
  \begin{align}
  \label{eqn:genbound}
    \sup\limits_{\wb \in \Wcal} |\hat{\Rcal}(\wb,S)-\Rcal(\wb)| &\leq  2B\sqrt{\frac{[\dimPH \Wcal+1] \log^2(nL^2)}{n} + \frac{ \log(7\Mut/\gamma)}{n}},
  \end{align}
with probability at least $1- \gamma$ over $S \sim \Dcal^{\otimes n}$. 
\end{prop}
\begin{proof}
By using the same proof technique as \cite[Theorem 2]{simsekli2020hausdorff}, we can show that \eqref{eqn:genbound} holds with $\dimBox \Wcal$ in place of $\dimPH \Wcal$. Since $\Wcal$ is bounded, we have $\dimBox \Wcal = \dimPH \Wcal$ by Theorem~\ref{thm:kozma}. %
The result follows.
\end{proof}
This result shows that the generalization error of the trajectories of a training algorithm is deeply linked to its topological properties as measured by the PH dimension. Thanks to novel connection, we have now access to the rich TDA toolbox, to be used for different purposes.

\subsection{Analyzing Deep Network Dynamics via Persistent Homology}

 \begin{algorithm2e} [t!]
 \DontPrintSemicolon
 \SetKwInOut{Input}{input}
 \SetKwInOut{Output}{output}
 \Input{The set of iterates $W=\{w_i\}_{i=1}^K$, smallest sample size $n_{\mathrm{min}}$, and a skip step $\Delta$, $\alpha$}
 \Output{$\dimPH W$}
 $n \gets n_{\mathrm{min}}$, \quad  $E \gets []$\\
 \While{$n\leq K$}{
    $W_n\gets \mathrm{sample}(W, n)$\,{\color{purple} \small \tcp{random sampling}}
    $\mathcal{W}_n \gets \VR(W_n)$\, {\color{purple} \small \tcp{Vietoris-Rips filtration}} 
    $E[i] \gets E_{\alpha}(\mathcal{W}_n) \triangleq \sum_{\gamma\in\PH_{0}(\mathcal{W}_n)} |I(\gamma)|^\alpha$\, {\color{purple} \small \tcp{compute lifetime sums from PH}}
    $n\gets n+\Delta$
    }
 $m, b \gets \mathrm{fitline}\left(\log(n_{\text{min}}:\Delta:K),\, \log(E)\right)$\, {\color{purple} \small \tcp{power law on $E^i_1(W)$}}
 $\dimPH W \gets \frac{\alpha}{1-m}$
 \caption{Computation of $\dimPH$.}
 \label{algo:dimPH}
 \end{algorithm2e}

By exploiting TDA tools, our goal in this section is to develop an algorithm to compute $\dimPH \Wcal$ for two main purposes. The first goal is to predict the generalization performance by using $\dimPH$. By this approach, we can use $\dimPH$ for hyperparameter tuning without having access to test data. The second goal is to incorporate $\dimPH$ as a regularizer to the optimization problem in order to improve generalization. Note that similar topological regularization strategies have already been proposed  \cite{bruel2019topology,chen2019topological} without a formal link to generalization. In this sense, our observations form the first step towards theoretically linking generalization and TDA.

In \cite{simsekli2020hausdorff}, to develop a computational approach, the authors first linked the intrinsic dimension to certain statistical properties of the underlying training algorithm, which can be then estimated. To do so, they required an additional topological regularity condition, which necessitates the existence of an `Ahlfors regular' measure defined on $\Wcal$, \ie, a finite Borel measure $\mu$ such that there exists $s,r_0>0$ where
$0<a r^{s} \leq \mu(B_r(x)) \leq b r^{s}<\infty$, 
holds for all $x \in \Wcal, 0<r \leq r_{0}$. This assumption was used to link the box dimension to another notion called Hausdorff dimension, which can be then linked to statistical properties of the training trajectories under further assumptions (see Section~\ref{sec:intro}). 
An interesting asset of our approach is that, 
we do not require this condition and thanks to the following result, we are able to develop an algorithm to directly estimate $\dimPH \Wcal$, while staying agnostic to the finer topological properties of $\Wcal$.

\begin{prop}
\label{prop:generalization}
Let $\Wcal \subset \mathbb{R}^d$ be a bounded set with $\dimPH \Wcal =: d^\star$. Then, for all $\veps >0$ and $\alpha \in (0, d^\star+\veps)$, there exists a constant $D_{\alpha,\veps}$, such that the following inequality holds for all $n \in \mathbb{N_+}$ and all collections $W_n = \{w_1, \dots, w_n\}$ with $w_i \in \Wcal$, $i=1,\dots,n$:
\begin{align}\label{eqn:persbound}
    E_{\alpha}^0(W_n) \leq D_{\alpha,\veps} n^{\frac{d^\star+\veps-\alpha}{d^\star+\veps}}.
\end{align}
\end{prop}
\begin{proof}
Since $\Wcal$ is bounded, we have $\dimBox \Wcal = d^\star$ by Theorem~\ref{thm:kozma}. Fix $\veps>0$.  Then, by Definition~\ref{def:mink}, there exists $\delta_0 = \delta_0(\veps) >0$ and a finite constant $C_\veps > 0$ such that for all $\delta\leq \delta_0$ the following inequality holds:
\begin{align}
    N_{\delta}(\Wcal) \leq C_\veps \delta^{-(d^\star+\veps)}.
\end{align}
Then, the result directly follows from \cite[Proposition 21]{schweinhart2020fractal}.
\end{proof}

This result suggests a simple strategy to estimate an upper bound of the intrinsic dimension from persistent homology. In particular, we note that rewriting \eqref{eqn:persbound} for logarithmic values give us that 
\begin{equation}
\left(1 - \frac{\alpha}{d^* + \epsilon}\right) \log n + \log D_{\alpha, \epsilon} \ge \log E_\alpha^0.    
\end{equation} 
If $\log E_\alpha^0$ and $\log n$ are sampled from the data and give an empirical slope $m$, then we see that $d^* + \epsilon \le \frac{m}{1 - \alpha}$. In many cases, we see that $d^* \approx \frac{\alpha}{1 - m}$ (as further explained in~\cref{subsec:ablate}), so we take $\frac{\alpha}{1 - m}$ as our PH dimension estimation. We provide the full algorithm for computing this from our sampled data in~\cref{algo:dimPH}. Note that our algorithm is similar to that proposed in \cite{Adams2020AFD}, although our method works for sets rather than probability measures. In our implementation we compute the homology by the celebrated Ripser package~\cite{ripser} unless otherwise specified. %

\paragraph{On computational complexity.}
Computing the Vietoris Rips complex is an active area of research, as the worst-case time complexity is meaningless due to natural sparsity~\cite{zomorodian2010fast}. Therefore, to calculate the time complexity of our estimator, we focus on analyzing the PH computation from the output simplices: calculating PH takes $O(p^w)$ time, where $w<2.4$ is the constant of matrix multiplication and $p$ is the number of simplices produced in the filtration~\cite{boissonnat2019}. Since we compute with $0^\text{th}$ order homology, this would imply that the computational complexity is $O(n^w)$, where $n$ is the number of points. In particular, this means that estimating the PH dimension would take $O(kn^w)$ time, where $k$ is the number of samples taken assuming that samples are evenly spaced in $[0,n]$.

\subsection{Regularizing Deep Networks via Persistent Homology}

Motivated by our results in~\cref{prop:generalization}, we theorize that controlling $\dim_{\mathrm{PH}} \mathcal{W}$ would help in reducing the generalization error. Towards this end, we develop a regularizer for our training procedure which seeks to minimize $\dim_{\mathrm{PH}} \mathcal{W}$ during train time. If we let $\mathcal{L}$ be our vanilla loss function, then we will instead optimize over our topological loss function $\mathcal{L}_\lambda := \mathcal{L} + \lambda \dim_{\mathrm{PH}} \mathcal{W}$, where $\lambda \ge 0$ controls the scale of the regularization and $\Wcal$ now denotes a sliding window of iterates (\eg, the latest $50$ iterates during training). This way, we aim to regularize the loss by considering the dimension of the ongoing training trajectory. 

In \cref{algo:dimPH}, we let $w_i$ be the stored weights from previous iterations for $i \in \{1, \dots, K - 1\}$ and let $w_K$ be the current weight iteration. Since the persistence diagram computation and linear regression are differentiable, this means that our estimate for $\dim_{\mathrm{PH}}$ is also differentiable, and, if $w_k$ is sampled as in~\cref{algo:dimPH}, is connected in the computation graph with $w_K$. We incorporate our regularizer into the network training using PyTorch~\cite{pytorch} and the associated persistent homology package \emph{torchph}~\cite{Hofer17a,Hofer19a}.
\insertimageC{1}{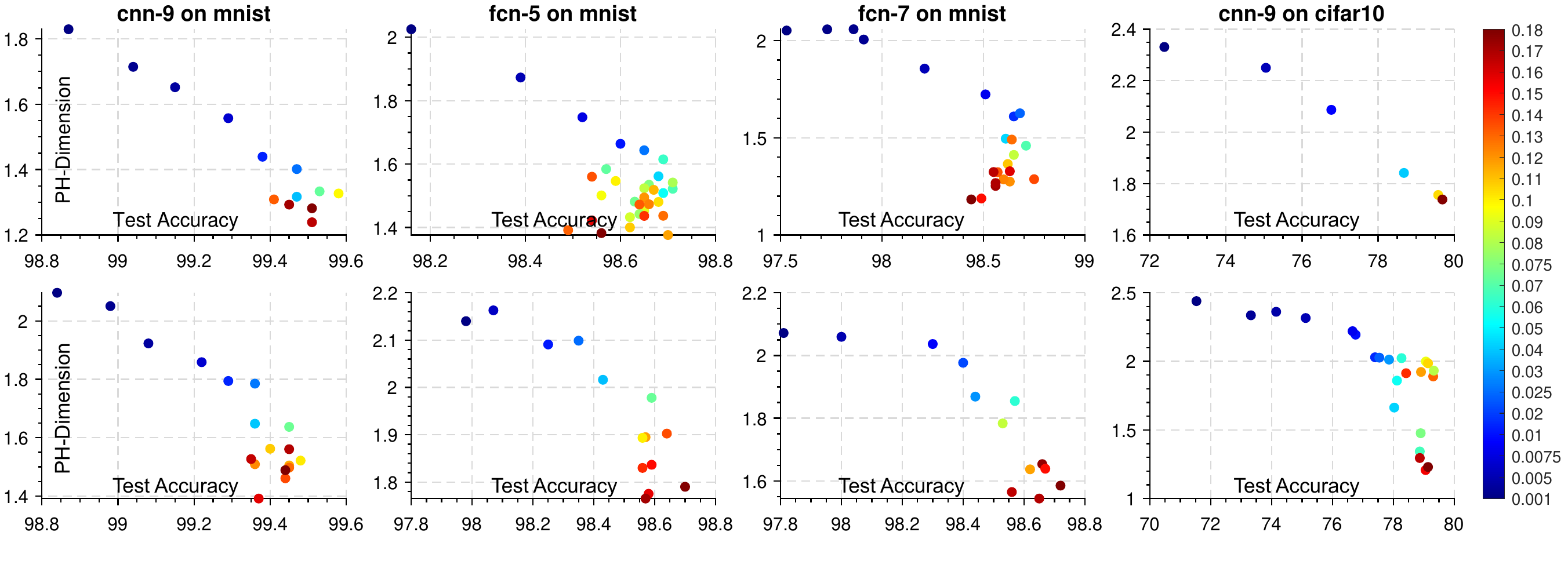}{PH-dimension vs test accuracy for different models and datasets. The rows correspond to the model and dataset, and the columns correspond to the batch size (50 and 100 for the top and bottom row respectively). The graphed points are marked with different colors corresponding to the learning rate. Note that the PH dimension is inversely correlated with test accuracy and is thus positively correlated with generalization error.}{fig:IDTA}{t!}
\section{Experimental Evaluations}
\label{sec:exp}

This section presents our experimental results in two parts: (i) analyzing and quantifying generalization in practical deep networks on real data, (ii) ablation studies on a random diffusion process. In all the experiments we will assume that the intrinsic dimension is strictly larger than $1$, hence we will set $\alpha=1$, unless specified otherwise. Further details are reported in the supplementary document. 

\subsection{Analyzing and Visualizing Deep Networks}

\paragraph{Measuring generalization.}

We first verify our main claim by showing that our persistent homology dimension derived from topological analysis of the training trajectories correctly measures of generalization. To demonstrate this, we apply our analysis to a wide variety of networks, training procedures, and hyperparameters. In particular, we train AlexNet~\cite{krizhevsky2012imagenet}, a 5-layer (fcn-5) and 7-layer (fcn-7) fully connected networks, and a 9-layer convolutional netowork (cnn-9) on MNIST, CIFAR10 and CIFAR100 datasets for multiple batch sizes and learning rates until convergence. For AlexNet, we consider 1000 iterates prior to convergence and, for the others, we only consider 200. Then, we estimate $\dimPH$ on the last iterates by using \cref{algo:dimPH}. For varying $n$, we randomly pick $n$ of last iterates and compute $E_\alpha^0$, and then we use the relation given in \eqref{eqn:persbound}.   

We obtain the ground truth (GT) generalization error as the gap between training and test accuracies.~\cref{fig:IDTA} plots the PH-dimension with respect to test accuracy and signals a strong correlation of our PH-dimension and actual performance gap. The lower the PH-dimension, the higher the test accuracy. Note that this results aligns well with that of~\cite{simsekli2020hausdorff}.
The figure also shows that the intrinsic dimensions across different datasets can be similar, even if the parameters of the models can vary greatly. This supports the recent hypothesis that what matters for the generalization is the effective capacity and not the parameter count. In fact, the dimension should be as minimal as possible without collapsing important representation features onto the same dimension.
The findings in~\cref{fig:IDTA} are further augmented with results in~\cref{fig:opt}, where a similar pattern is observed on AlexNet and CIFAR100.

\insertimageC{1}{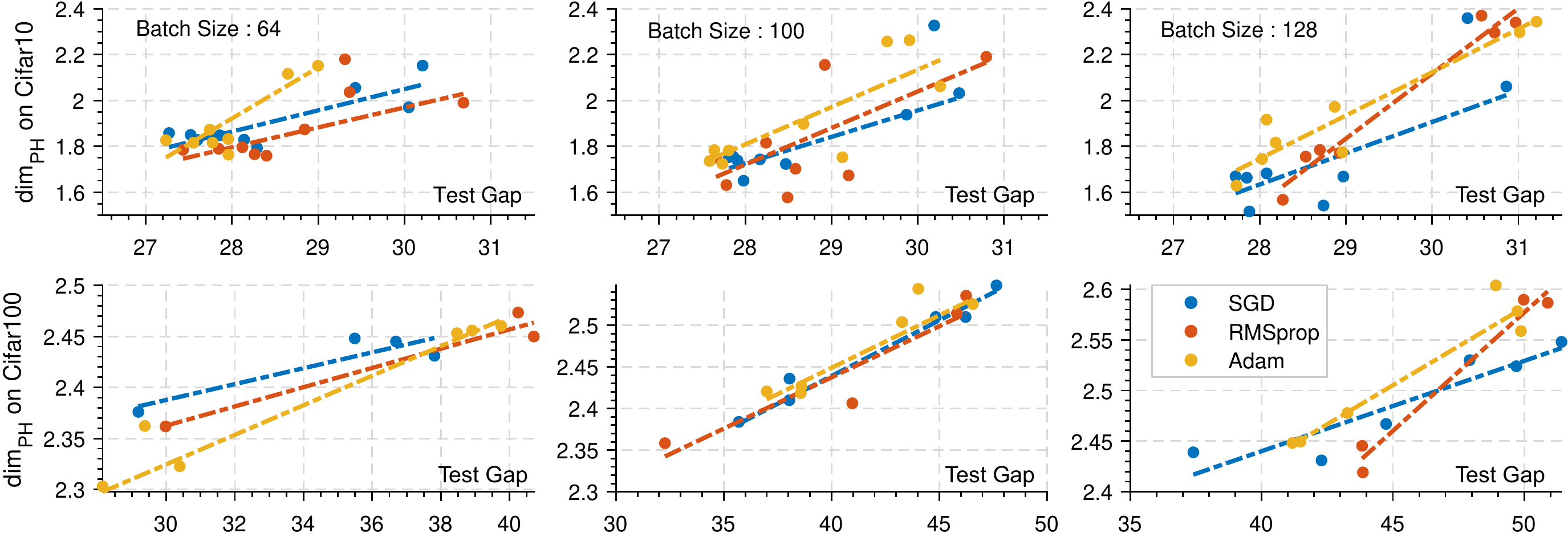}{(Estimated) persistent homology dimension vs generalization error (training accuracy - test accuracy) for different datasets (top row CIFAR10, bottom row CIFAR100) and optimizers on AlexNet. We plot the data points and lines of best fit. Note that the PH dimension is \emph{directly} correlated with the generalization error and is consistent across datasets and optimizers.\vspace{-1mm}}{fig:opt}{t!}

\paragraph{Can $\dimPH$ capture intrinsic properties of trajectories?}
After revealing that our ID estimation is a gauge for generalization, we set out to investigate whether it really hinges on the intrinsic properties of the data. We train several instances of 7-fcn for different learning rates and batch sizes. We compute the PH-dimension of each network using training trajectories. 
We visualize the following in the rows of~\cref{fig:cifar10fcn7} sorted by $\dimPH$: (i) $200\times 200$ distance matrix of the sequence of iterates $w_1,\dots,w_K$ (which is the basis for PH computations), (ii) corresponding $\log E_{\alpha=1}^0$ estimates as we sweep over $n$ in an increasing fashion,  (iii) persistence diagrams per each distance matrix. It is clear that there is a strong correlation between $\dimPH$ and the structure of the distance matrix. As dimension increases, matrix of distances become non-uniformly \emph{pixelated}. The slope estimated from the total edge lengths the second row is a quantity proportional to our dimension. Note that the slope decreases as our estimte increases (hence generalization tends to decrease). We further observe clusters emerging in the persistence diagram. The latter has also been reported for better generalizing networks, though using a different notion of a topological space~\cite{bruel2019topology}.

\insertimageC{1}{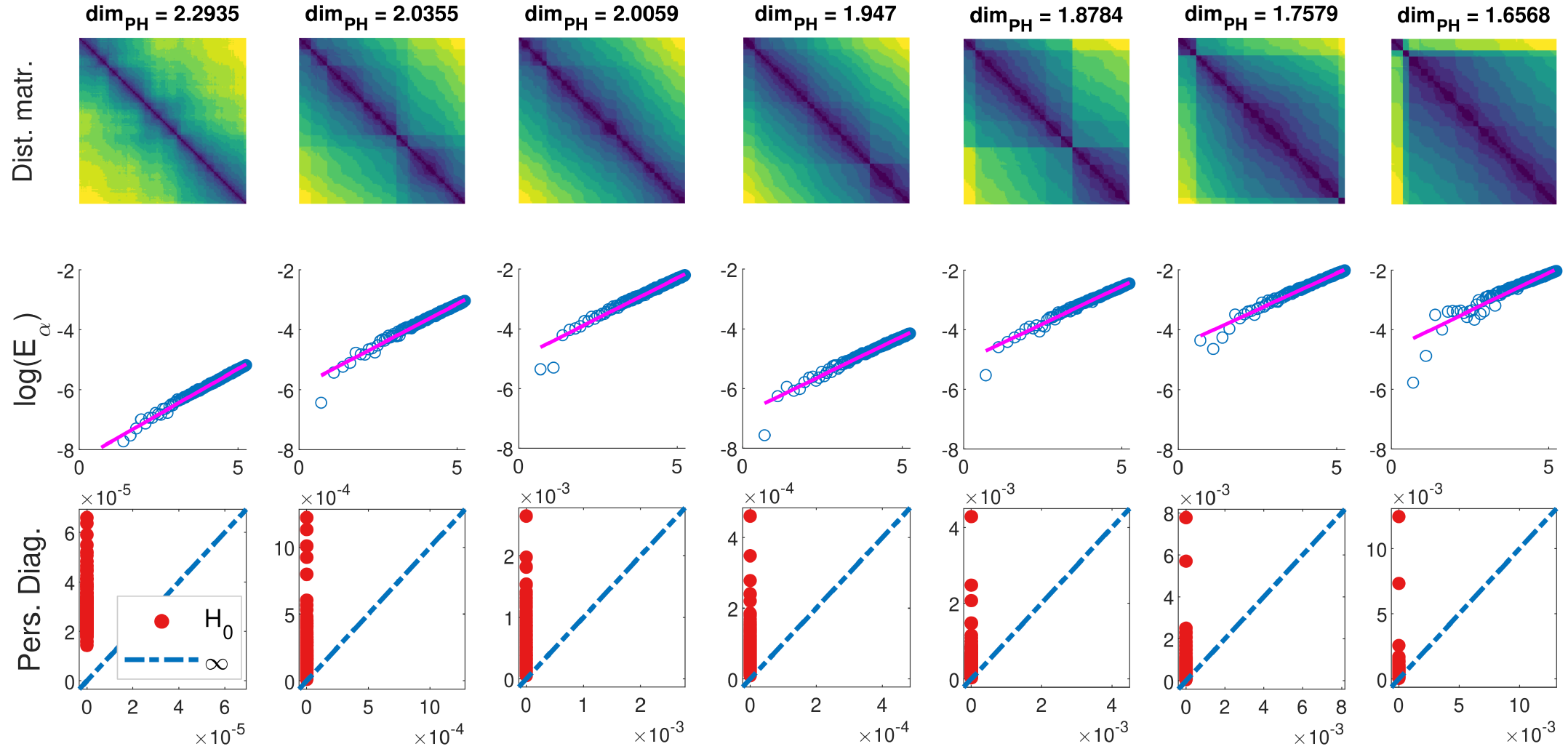}{We visualize topological information for a 7-layer fully connected network on CIFAR10 data. In the top row, we visualize the distance matrices computed between network weights corresponding to the last $200$ iterations of training. In the middle, we visualize the corresponding behavior of our estimator as we increase the number of samples. In the bottom row, we visualize the 0-th order persistent diagrams for the full data. As our PH dimension decreases, the matrix becomes more segmented, the estimator slope decreases, and the persistent diagram becomes sparser. We provide more information about these results in the supplement. \vspace{-3mm}}{fig:cifar10fcn7}{t!}

\paragraph{Is $\dimPH$ a real indicator of generalization?}
To quantitatively assess the quality of our complexity measure, we gather two statistics: (i) we report the average $p$-value over different batch sizes for AlexNet trained with SGD on the Cifar100 dataset. The value of $p=0.0157<0.05$ confirms the statistical significance. 
Next, we follow the recent literature~\cite{jiang2020neurips} and consult the Kendall correlation coefficient (KCC). Similar to the $p$-value experiment above, we compute KCC for AlexNet+SGD for different batch sizes $(64, 100, 128)$ and attain $(0.933, 0.357, 0.733)$ respectively. Note that, a positive correlation signals that the test gap closes as $\dimPH$ decreases. Both of these experiments agree with our theoretical insights that connect generalization to a topological characteristic of a neural network: intrinsic dimension of training trajectories. 
\vspace{-1mm}
\paragraph{Effect of different training algorithms.}
We also verify that  
our method is algorithm-agnostic and does not require assumptions on the training algorithm. In particular, we show that our above analyses extend to both the RMSProp \cite{Tieleman2012} and Adam \cite{Kingma2015AdamAM} optimizer.
Our results are visualized in~\cref{fig:opt}. We plot the dimension with respect to the generalization error for varying optimizers and batch sizes; our results verify that the generalization error (which is inversely related to the test accuracy) is positively correlated with the PH dimension. This corroborates our previous results in~\cref{fig:IDTA} and in particular shows that our dimension estimator of test gap is indeed algorithm-agnostic.

\begin{wrapfigure}[15]{r}{0.42\textwidth} %
\vspace{-7mm}
\includegraphics[width=0.41\textwidth]{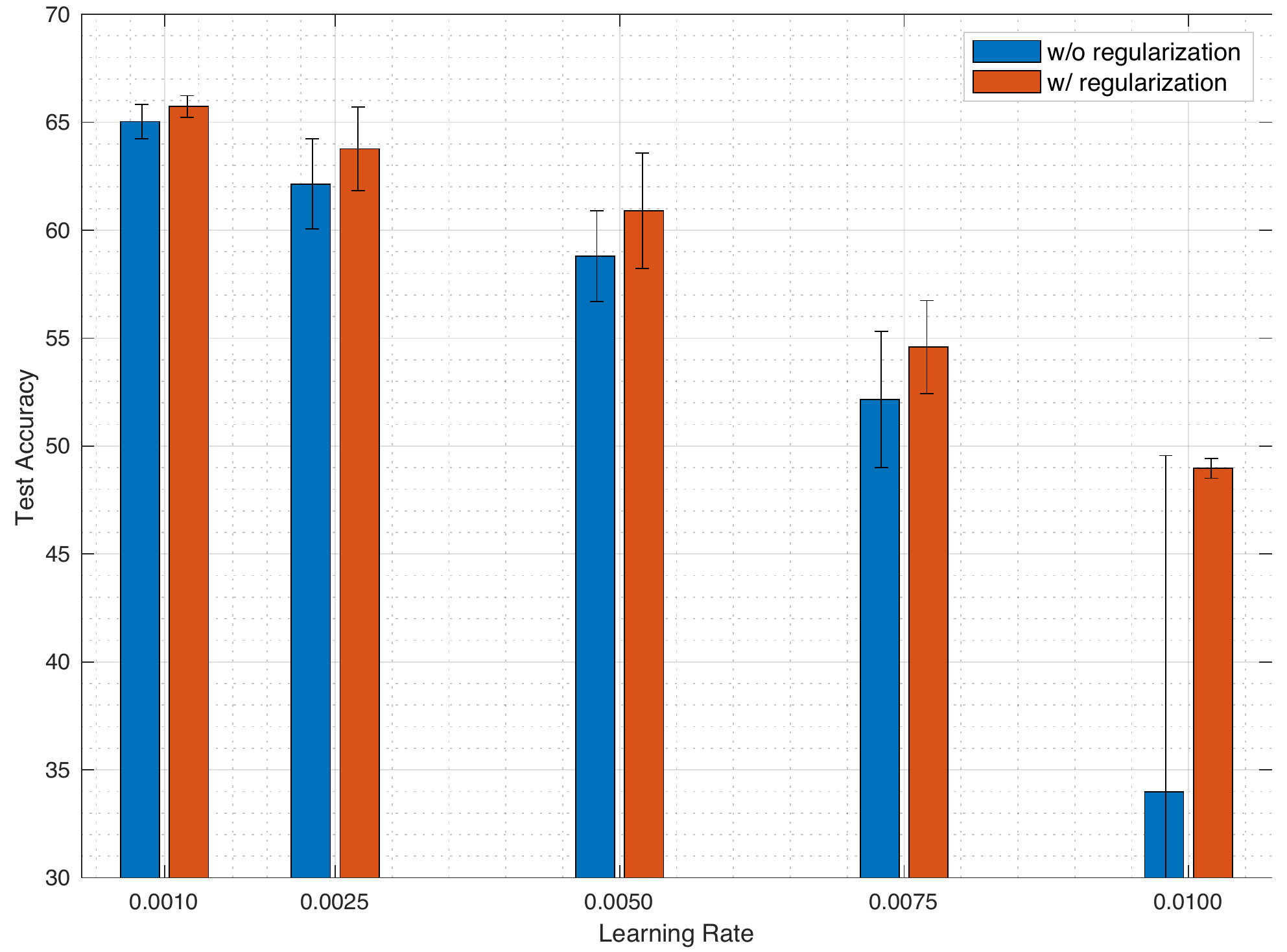}
\vspace{-3mm}
\caption{\small Effect of regularization on test accuracy for various learning rates. Our regularization is consistently able to produce higher accuracies, and this effect is more pronounced when the network has a lower test accuracy.}
\label{fig:reg}
\end{wrapfigure}
\vspace{-2mm}\paragraph{Encouraging generalization via regularization $\dimPH$.}
We furthermore verify that our topological regularizer is able to help control the test gap in accordance with our theory. We train a Lenet-5 network \cite{LeCun1998GradientbasedLA} on Cifar10 \cite{Krizhevsky2009LearningML} and compare a clean trianing with a training with our topological regularizer with $\lambda$ set to $1$. We train for $200$ epochs with a batch size of $128$ and report the train and test accuracies in~$\cref{fig:reg}$ over a variety of learning rates. We tested over $10$ trials and found that, with $p < 0.05$ for all cases except $\mathrm{lr}=0.01$, the results are different.

Our topological optimizer is able to produce the best improvements when our network is not able to converge well. These results show that our regularizer behaves as expected: the regularizer is able to recover poor training dynamics. We note that this experiment uses a simple architecture and as such, it presents a proof of concept. We do not aim for the state of the art results. Furthermore, we directly compared our approach with the generalization estimator of \cite{Corneanu2020ComputingTT}, which most closely resembles our construction. In particular, we found their method does not scale and is often numerically unreliable. For example, their methodology grows quadratically with respect to number of network weights and linearly with the dataset size, while our method does not scale much beyond memory usage with vectorized computation. Furthermore, for many of our test networks, their metric space construction (which is based off of the correlation between activation functions and used for the Vietoris-Rips complex) would be numerically brittle and result in degenerate persistent homology. These prevent~\cite{Corneanu2020ComputingTT} to be applicable in this scenario.

\subsection{Ablation Studies}\label{subsec:ablate}
To assess the quality of our dimension estimator, we now perform ablation studies, on a synthetic data whose the ground truth ID is known. To this end, we use the synthetic experimental setting presented in \cite{simsekli2020hausdorff} (see the supplementary document for details),
and we simulate a $d=128$ dimensional stable Levy process with varying number of points $100\leq n\leq 1500$ and tail indices $1\leq\beta\leq 2$. Note that the tail index equals the intrinsic dimension in this case, which is an order of magnitude lower for this experiment.\vspace{-1mm}
\paragraph{Can $\dimPH$ match the ground truth ID?}
We first try to predict the GT intrinsic dimension running~\cref{algo:dimPH} on this data. We also estimate the TwoNN dimension~\cite{facco2017estimating} to quantify how the state of the art ID estimators correlate with GT in such heavy tailed regime. Our results are plotted in~\cref{fig:diffplots}. Note that as $n$ increases our estimator becomes smoother and well approximates the GT up to a slight over-estimation, a repeatedly observed phenomenon~\cite{campadelli2015intrinsic}. TwoNN does not guarantee recovering the box-dimension. While it is found to be useful in estimating the ID of data~\cite{ansuini2019intrinsic}, we find it to be less desirable in a heavy-tailed regime as reflected in the plots. Our supplementary material provides further results on other, non-dynamics like synthetic dataset such as points on a sphere where TwoNN can perform better. We also include a robust line fitting variant of our approach $\PH_0$-RANSAC, where a random sample consensus is applied iteratively. Though, as our data is not outlier-corrupted, we do not observe a large improvement.\vspace{-1mm}

\insertimageC{1}{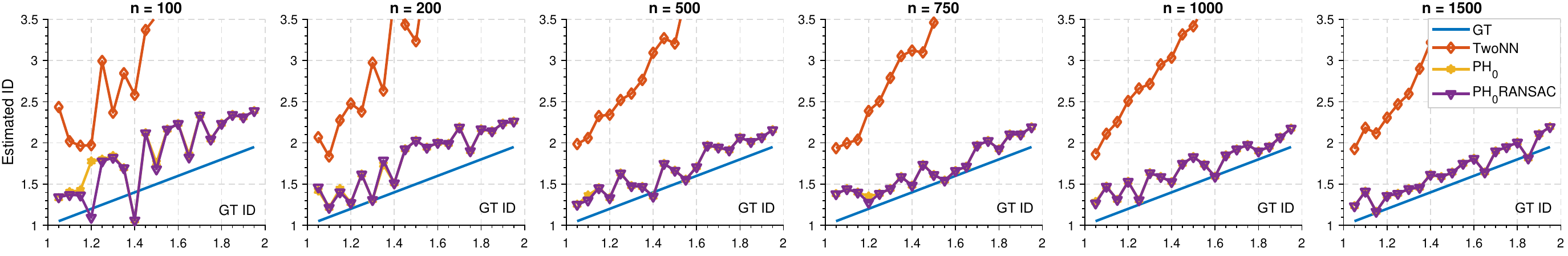}{Estimated intrinsic dimension vs ground-truth intrinsic dimension for different dimension estimators on synthetic diffusion data. Our $\mathrm{PH}_0$ (yellow) and $\mathrm{PH}_0\mathrm{RANSAC}$ (purple) estimators coincide as the linear regression step of our computation is well behaved. We note that our persistent homology dimension estimation is able to accurately recover the ground truth.\vspace{-3mm}}{fig:diffplots}{t!}

\begin{wrapfigure}[15]{r}{0.375\textwidth}
    \vspace{-4mm}
\includegraphics[width=0.375\textwidth]{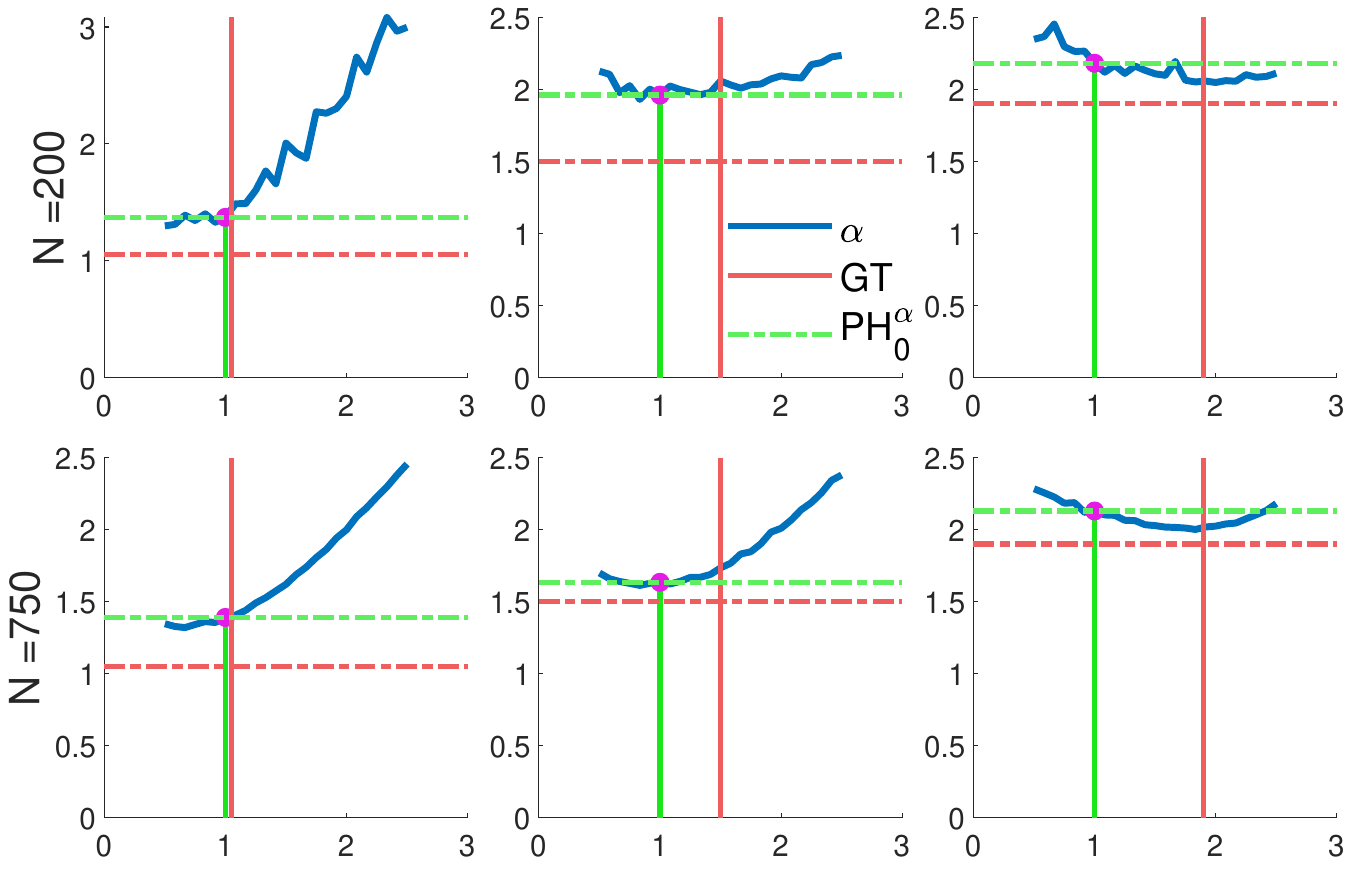}\vspace{-5mm}
    \caption{\small $\dimPH$ estimate versus various $\alpha$ on the synthetic diffusion data. Our estimate of $\alpha=1$ provides a very good estimate for a wide variety of intrinsic dimensions.}
    \label{fig:alphaDiff}
\end{wrapfigure}
\paragraph{Effect of $\alpha$ on dimension estimation.}
While our theory requires $\alpha$ to be smaller than the intrinsic dimension of the trajectories, in all of our experiments we fix $\alpha = 1.0$. It is of curiosity whether such choice hampers our estimates. To see the effect, we vary $\alpha$ in range $[0.5,2.5]$ and plot our estimates in~\cref{fig:alphaDiff}. It is observed (blue curve) that our dimension estimate follows a U-shaped trend with increasing $\alpha$. We indicate the GT ID by a dashed red line and our estimate as a dashed green line. Ideally, these two horizontal lines should overlap. It is noticeable that, given the oracle for GT ID, it might be possible to \emph{optimize for} an $\alpha^\star$. Yet, such information is not available for the deep networks. Nevertheless, $\alpha=1$ seems to yield reasonable performance and we leave the estimation of a better $\alpha$ for future work. We provide additional results in our supplementary material.

\section{Conclusion}
\label{sec:conclude}

In this paper, we developed novel connections between $\dimPH$ of the training trajectory and the generalization error. Using these insights, we proposed a method for estimating the $\dimPH$ from data and, unlike previous work \cite{simsekli2020hausdorff}, our approach 
does not presuppose any conditions on the trajectory and offers a simple algorithm. By leveraging the differentiability of PH computation, we showed that we can use $\dimPH$ as a regularizer during training, which improved the performance in different setups.

\textbf{Societal Impact and Limitations. } 
We believe that our study will not pose any negative societal or ethical consequences due to its theoretical nature. 
The main limitation of our study is that it solely considers the terms $E_\alpha^0$, whereas PH offers a much richer structure. Hence, as our next step, we will explore finer ways to incorporate PH in generalization performance. We will further extend our results in terms of dimensions of measures by using the techniques presented in  \cite{camuto2021fractal}.

\newpage

\newpage

\section*{Acknowledgements}
  Umut \c{S}im\c{s}ekli's research is supported by the French government under management of Agence Nationale de la Recherche as part of the ``Investissements d’avenir'' program, reference ANR-19-P3IA-0001 (PRAIRIE 3IA Institute).

\bibliographystyle{alpha}

\begin{thebibliography}{FGFAEV21}

\bibitem[AAF{\etalchar{+}}20]{Adams2020AFD}
Henry Adams, M.~Aminian, Elin Farnell, M.~Kirby, C.~Peterson, Joshua Mirth,
  R.~Neville, P.~Shipman, and C.~Shonkwiler.
\newblock A fractal dimension for measures via persistent homology.
\newblock {\em arXiv: Dynamical Systems}, pages 1--31, 2020.

\bibitem[AGNZ18]{arora2018stronger}
Sanjeev Arora, Rong Ge, Behnam Neyshabur, and Yi~Zhang.
\newblock Stronger generalization bounds for deep nets via a compression
  approach.
\newblock In {\em Proceedings of the 35th International Conference on Machine
  Learning}, volume~80, pages 254--263. PMLR, 10--15 Jul 2018.

\bibitem[ALMZ19]{ansuini2019intrinsic}
Alessio Ansuini, Alessandro Laio, Jakob~H Macke, and Davide Zoccolan.
\newblock Intrinsic dimension of data representations in deep neural networks.
\newblock In {\em Proceedings of the Advances in Neural Information Processing
  Systems (NeurIPS)}, volume~32. Curran Associates, Inc., 2019.

\bibitem[Bau21]{ripser}
Ulrich Bauer.
\newblock Ripser: efficient computation of vietoris-rips persistence barcodes,
  February 2021.
\newblock Preprint.

\bibitem[BE02]{bousquet2002stability}
Olivier Bousquet and Andr{\'e} Elisseeff.
\newblock Stability and generalization.
\newblock {\em JMLR}, 2(Mar), 2002.

\bibitem[Ben69]{bennett1969intrinsic}
Robert Bennett.
\newblock The intrinsic dimensionality of signal collections.
\newblock {\em IEEE Transactions on Information Theory}, 15(5):517--525, 1969.

\bibitem[BG60]{blumenthal1960some}
Robert~M Blumenthal and Ronald~K Getoor.
\newblock Some theorems on stable processes.
\newblock {\em Transactions of the American Mathematical Society},
  95(2):263--273, 1960.

\bibitem[BGND{\etalchar{+}}19]{bruel2019topology}
Rickard Br{\"u}el-Gabrielsson, Bradley~J Nelson, Anjan Dwaraknath, Primoz
  Skraba, Leonidas~J Guibas, and Gunnar Carlsson.
\newblock A topology layer for machine learning.
\newblock {\em arXiv preprint arXiv:1905.12200}, 2019.

\bibitem[BM02]{bartlett2002rademacher}
Peter~L Bartlett and Shahar Mendelson.
\newblock Rademacher and gaussian complexities: Risk bounds and structural
  results.
\newblock {\em Journal of Machine Learning Research}, 3(Nov):463--482, 2002.

\bibitem[BO18]{blier2018description}
L\'{e}onard Blier and Yann Ollivier.
\newblock The description length of deep learning models.
\newblock In S.~Bengio, H.~Wallach, H.~Larochelle, K.~Grauman, N.~Cesa-Bianchi,
  and R.~Garnett, editors, {\em Proceedings of the Advances in Neural
  Information Processing Systems (NeurIPS)}, volume~31. Curran Associates,
  Inc., 2018.

\bibitem[BP19]{boissonnat2019}
Jean-Daniel Boissonnat and Siddharth Pritam.
\newblock {Computing Persistent Homology of Flag Complexes via Strong
  Collapses}.
\newblock In {\em 35th International Symposium on Computational Geometry (SoCG
  2019)}, volume 129 of {\em Leibniz International Proceedings in Informatics
  (LIPIcs)}, pages 55:1--55:15, 2019.

\bibitem[BSE{\etalchar{+}}21]{barsbey2021heavy}
Melih Barsbey, Milad Sefidgaran, Murat~A Erdogdu, Ga{\"e}l Richard, and Umut
  {\c{S}}im{\c{s}}ekli.
\newblock Heavy tails in sgd and compressibility of overparametrized neural
  networks.
\newblock In {\em NeurIPS}, 2021.

\bibitem[BT74]{beaton1974fitting}
Albert~E Beaton and John~W Tukey.
\newblock The fitting of power series, meaning polynomials, illustrated on
  band-spectroscopic data.
\newblock {\em Technometrics}, 16(2):147--185, 1974.

\bibitem[Car14]{carlsson2014topological}
Gunnar Carlsson.
\newblock Topological pattern recognition for point cloud data.
\newblock {\em Acta Numerica}, 23:289--368, 2014.

\bibitem[CCCR15]{campadelli2015intrinsic}
P~Campadelli, E~Casiraghi, C~Ceruti, and A~Rozza.
\newblock Intrinsic dimension estimation: Relevant techniques and a benchmark
  framework.
\newblock {\em Mathematical Problems in Engineering}, 2015, 2015.

\bibitem[CDE{\etalchar{+}}21]{camuto2021fractal}
Alexander Camuto, George Deligiannidis, Murat~A Erdogdu, Mert
  G{\"u}rb{\"u}zbalaban, Umut {\c{S}}im{\c{s}}ekli, and Lingjiong Zhu.
\newblock Fractal structure and generalization properties of stochastic
  optimization algorithms.
\newblock In {\em NeurIPS}, 2021.

\bibitem[CH03]{costa2003manifold}
Jose Costa and Alfred Hero.
\newblock Manifold learning with geodesic minimal spanning trees.
\newblock {\em arXiv preprint cs/0307038}, 2003.

\bibitem[CHN19]{Hofer19a}
M.~Dixit C.~Hofer, R.~Kwitt and M.~Niethammer.
\newblock Connectivity-optimized representation learning via persistent
  homology.
\newblock In {\em ICML}, 2019.

\bibitem[CHU17]{Hofer17a}
M.~Niethammer C.~Hofer, R.~Kwitt and A.~Uhl.
\newblock Deep learning with topological signatures.
\newblock In {\em NIPS}, 2017.

\bibitem[CMEM20]{Corneanu2020ComputingTT}
C.~Corneanu, M.~Madadi, S.~Escalera, and A.~Mart{\'i}nez.
\newblock Computing the testing error without a testing set.
\newblock {\em 2020 IEEE/CVF Conference on Computer Vision and Pattern
  Recognition (CVPR)}, pages 2674--2682, 2020.

\bibitem[CNBW19]{chen2019topological}
Chao Chen, Xiuyan Ni, Qinxun Bai, and Yusu Wang.
\newblock A topological regularizer for classifiers via persistent homology.
\newblock In {\em The 22nd International Conference on Artificial Intelligence
  and Statistics}, pages 2573--2582. PMLR, 2019.

\bibitem[CS16]{camastra2016intrinsic}
Francesco Camastra and Antonino Staiano.
\newblock Intrinsic dimension estimation: Advances and open problems.
\newblock {\em Information Sciences}, 328:26--41, 2016.

\bibitem[CWZ{\etalchar{+}}21]{camuto2021asymmetric}
Alexander Camuto, Xiaoyu Wang, Lingjiong Zhu, Chris Holmes, Mert
  G{\"u}rb{\"u}zbalaban, and Umut {\c{S}}im{\c{s}}ekli.
\newblock Asymmetric heavy tails and implicit bias in gaussian noise
  injections.
\newblock In {\em ICML}, 2021.

\bibitem[DBI18]{deng2018ppfnet}
Haowen Deng, Tolga Birdal, and Slobodan Ilic.
\newblock Ppfnet: Global context aware local features for robust 3d point
  matching.
\newblock In {\em Proceedings of the IEEE conference on computer vision and
  pattern recognition}, pages 195--205, 2018.

\bibitem[DCLT18]{devlin2018bert}
Jacob Devlin, Ming-Wei Chang, Kenton Lee, and Kristina Toutanova.
\newblock Bert: Pre-training of deep bidirectional transformers for language
  understanding.
\newblock {\em arXiv preprint arXiv:1810.04805}, 2018.

\bibitem[DZF19]{Dollr2019WhatDI}
P.~Doll{\'a}r, C.~L. Zitnick, and P.~Frossard.
\newblock What does it mean to learn in deep networks ? and , how does one
  detect adversarial attacks ?
\newblock 2019.

\bibitem[EH10]{edelsbrunner2010computational}
Herbert Edelsbrunner and John Harer.
\newblock {\em Computational topology: an introduction}.
\newblock American Mathematical Soc., 2010.

\bibitem[EMH{\etalchar{+}}19]{elsken2019neural}
Thomas Elsken, Jan~Hendrik Metzen, Frank Hutter, et~al.
\newblock Neural architecture search: A survey.
\newblock {\em Journal of Machine Learning Research}, 20(55):1--21, 2019.

\bibitem[Fal04]{falconer2004fractal}
Kenneth Falconer.
\newblock {\em Fractal geometry: mathematical foundations and applications}.
\newblock John Wiley \& Sons, 2004.

\bibitem[FB81]{fischler1981random}
Martin~A Fischler and Robert~C Bolles.
\newblock Random sample consensus: a paradigm for model fitting with
  applications to image analysis and automated cartography.
\newblock {\em Communications of the ACM}, 24(6):381--395, 1981.

\bibitem[FC18]{frankle2018lottery}
Jonathan Frankle and Michael Carbin.
\newblock The lottery ticket hypothesis: Finding sparse, trainable neural
  networks.
\newblock In {\em International Conference on Learning Representations}, 2018.

\bibitem[FdRL17]{facco2017estimating}
Elena Facco, Maria d’Errico, Alex Rodriguez, and Alessandro Laio.
\newblock Estimating the intrinsic dimension of datasets by a minimal
  neighborhood information.
\newblock {\em Scientific reports}, 7(1):1--8, 2017.

\bibitem[FGFAEV21]{fernandez2021determining}
David~P{\'e}rez Fern{\'a}ndez, Asier Guti{\'e}rrez-Fandi{\~n}o, Jordi
  Armengol-Estap{\'e}, and Marta Villegas.
\newblock Determining structural properties of artificial neural networks using
  algebraic topology.
\newblock {\em arXiv preprint arXiv:2101.07752}, 2021.

\bibitem[Ghr10]{ghrist2010configuration}
Robert Ghrist.
\newblock Configuration spaces, braids, and robotics.
\newblock In {\em Braids: Introductory Lectures on Braids, Configurations and
  Their Applications}, pages 263--304. World Scientific, 2010.

\bibitem[GJ16]{gao2016degrees}
Tianxiang Gao and Vladimir Jojic.
\newblock Degrees of freedom in deep neural networks.
\newblock UAI'16, page 232–241. AUAI Press, 2016.

\bibitem[GLW{\etalchar{+}}21]{gojcic2020weakly}
Zan Gojcic, Or~Litany, Andreas Wieser, Leonidas~J. Guibas, and Tolga Birdal.
\newblock Weakly supervised learning of rigid 3d scene flow.
\newblock In {\em Proceedings of the IEEE/CVF Conference on Computer Vision and
  Pattern Recognition (CVPR)}, pages 5692--5703, June 2021.

\bibitem[GP04]{grassberger2004measuring}
Peter Grassberger and Itamar Procaccia.
\newblock Measuring the strangeness of strange attractors.
\newblock In {\em The Theory of Chaotic Attractors}, pages 170--189. Springer,
  2004.

\bibitem[GSZ21]{gurbuzbalaban2021heavy}
Mert Gurbuzbalaban, Umut Simsekli, and Lingjiong Zhu.
\newblock The heavy-tail phenomenon in sgd.
\newblock In {\em International Conference on Machine Learning}, pages
  3964--3975. PMLR, 2021.

\bibitem[Hau18]{hausdorff1918dimension}
Felix Hausdorff.
\newblock Dimension und {\"a}u{\ss}eres ma{\ss}.
\newblock {\em Mathematische Annalen}, 79(1):157--179, 1918.

\bibitem[HJTW21]{hsu2021generalization}
Daniel Hsu, Ziwei Ji, Matus Telgarsky, and Lan Wang.
\newblock Generalization bounds via distillation.
\newblock In {\em International Conference on Learning Representations}, 2021.

\bibitem[HM20]{hodgkinson2020multiplicative}
Liam Hodgkinson and Michael~W. Mahoney.
\newblock Multiplicative noise and heavy tails in stochastic optimization.
\newblock {\em arXiv:2006.06293 [cs, math, stat]}, June 2020.

\bibitem[Hub92]{huber1992robust}
Peter~J Huber.
\newblock Robust estimation of a location parameter.
\newblock In {\em Breakthroughs in statistics}, pages 492--518. Springer, 1992.

\bibitem[Ish93]{isham1993statistical}
Valerie Isham.
\newblock Statistical aspects of chaos: a review.
\newblock {\em Networks and Chaos-Statistical and Probabilistic Aspects}, pages
  124--200, 1993.

\bibitem[JFH15]{janson2015effective}
Lucas Janson, William Fithian, and Trevor~J Hastie.
\newblock Effective degrees of freedom: a flawed metaphor.
\newblock {\em Biometrika}, 102(2):479--485, 2015.

\bibitem[JFY{\etalchar{+}}20]{jiang2020neurips}
Yiding Jiang, Pierre Foret, Scott Yak, Daniel~M Roy, Hossein Mobahi,
  Gintare~Karolina Dziugaite, Samy Bengio, Suriya Gunasekar, Isabelle Guyon,
  and Behnam Neyshabur.
\newblock Neurips 2020 competition: Predicting generalization in deep learning
  (version 1.1).
\newblock Technical report, Technical Report Preprint: December 16, 2020.

\bibitem[Jol86]{jolliffe1986principal}
Ian~T Jolliffe.
\newblock Principal components in regression analysis.
\newblock In {\em Principal component analysis}, pages 129--155. Springer,
  1986.

\bibitem[KB15]{Kingma2015AdamAM}
Diederik~P. Kingma and Jimmy Ba.
\newblock Adam: A method for stochastic optimization.
\newblock {\em CoRR}, abs/1412.6980, 2015.

\bibitem[K{\'e}g02]{kegl2002intrinsic}
Bal{\'a}zs K{\'e}gl.
\newblock Intrinsic dimension estimation using packing numbers.
\newblock In {\em NIPS}, pages 681--688. Citeseer, 2002.

\bibitem[KLS06]{kozma2006minimal}
Gady Kozma, Zvi Lotker, and Gideon Stupp.
\newblock The minimal spanning tree and the upper box dimension.
\newblock {\em Proceedings of the American Mathematical Society},
  134(4):1183--1187, 2006.

\bibitem[Kri09]{Krizhevsky2009LearningML}
A.~Krizhevsky.
\newblock Learning multiple layers of features from tiny images.
\newblock 2009.

\bibitem[KSH12]{krizhevsky2012imagenet}
Alex Krizhevsky, Ilya Sutskever, and Geoffrey~E Hinton.
\newblock Imagenet classification with deep convolutional neural networks.
\newblock {\em Advances in neural information processing systems},
  25:1097--1105, 2012.

\bibitem[LB05]{levina2005maximum}
Elizaveta Levina and Peter~J Bickel.
\newblock Maximum likelihood estimation of intrinsic dimension.
\newblock In {\em Advances in neural information processing systems}, pages
  777--784, 2005.

\bibitem[LBBH98]{LeCun1998GradientbasedLA}
Y.~LeCun, L.~Bottou, Yoshua Bengio, and P.~Haffner.
\newblock Gradient-based learning applied to document recognition.
\newblock 1998.

\bibitem[LFLY18]{li2018measuring}
Chunyuan Li, Heerad Farkhoor, Rosanne Liu, and Jason Yosinski.
\newblock Measuring the intrinsic dimension of objective landscapes.
\newblock In {\em International Conference on Learning Representations}, 2018.

\bibitem[LZ08]{lin2008riemannian}
Tong Lin and Hongbin Zha.
\newblock Riemannian manifold learning.
\newblock {\em IEEE Transactions on Pattern Analysis and Machine Intelligence},
  30(5):796--809, 2008.

\bibitem[MM19]{martin2019traditional}
Charles~H Martin and Michael~W Mahoney.
\newblock Traditional and heavy-tailed self regularization in neural network
  models.
\newblock In {\em International Conference on Machine Learning (ICML)}, 2019.

\bibitem[MWH{\etalchar{+}}18]{ma2018dimensionality}
Xingjun Ma, Yisen Wang, Michael~E Houle, Shuo Zhou, Sarah Erfani, Shutao Xia,
  Sudanthi Wijewickrema, and James Bailey.
\newblock Dimensionality-driven learning with noisy labels.
\newblock In {\em International Conference on Machine Learning}, pages
  3355--3364. PMLR, 2018.

\bibitem[NBMS17]{neyshabur2017}
Behnam Neyshabur, Srinadh Bhojanapalli, David Mcallester, and Nati Srebro.
\newblock Exploring generalization in deep learning.
\newblock In I.~Guyon, U.~V. Luxburg, S.~Bengio, H.~Wallach, R.~Fergus,
  S.~Vishwanathan, and R.~Garnett, editors, {\em Advances in Neural Information
  Processing Systems}, volume~30. Curran Associates, Inc., 2017.

\bibitem[PGM{\etalchar{+}}19]{pytorch}
Adam Paszke, Sam Gross, Francisco Massa, Adam Lerer, James Bradbury, Gregory
  Chanan, Trevor Killeen, Zeming Lin, Natalia Gimelshein, Luca Antiga, Alban
  Desmaison, Andreas Kopf, Edward Yang, Zachary DeVito, Martin Raison, Alykhan
  Tejani, Sasank Chilamkurthy, Benoit Steiner, Lu~Fang, Junjie Bai, and Soumith
  Chintala.
\newblock Pytorch: An imperative style, high-performance deep learning library.
\newblock In H.~Wallach, H.~Larochelle, A.~Beygelzimer, F.~d\textquotesingle
  Alch\'{e}-Buc, E.~Fox, and R.~Garnett, editors, {\em Advances in Neural
  Information Processing Systems 32}, pages 8024--8035. Curran Associates,
  Inc., 2019.

\bibitem[QSMG17]{qi2017pointnet}
Charles~R Qi, Hao Su, Kaichun Mo, and Leonidas~J Guibas.
\newblock Pointnet: Deep learning on point sets for 3d classification and
  segmentation.
\newblock In {\em Proceedings of the IEEE conference on computer vision and
  pattern recognition}, pages 652--660, 2017.

\bibitem[RB21]{reani2021cycle}
Yohai Reani and Omer Bobrowski.
\newblock Cycle registration in persistent homology with applications in
  topological bootstrap.
\newblock {\em arXiv preprint arXiv:2101.00698}, 2021.

\bibitem[RBH{\etalchar{+}}21]{rempe2021humor}
Davis Rempe, Tolga Birdal, Aaron Hertzmann, Jimei Yang, Srinath Sridhar, and
  Leonidas~J. Guibas.
\newblock Humor: 3d human motion model for robust pose estimation.
\newblock In {\em Proceedings of the IEEE/CVF International Conference on
  Computer Vision (ICCV)}, pages 11488--11499, October 2021.

\bibitem[RTB{\etalchar{+}}19]{Rieck2019NeuralPA}
Bastian~Alexander Rieck, Matteo Togninalli, Christian Bock, Michael Moor, Max
  Horn, Thomas Gumbsch, and K.~Borgwardt.
\newblock Neural persistence: A complexity measure for deep neural networks
  using algebraic topology.
\newblock {\em ArXiv}, abs/1812.09764, 2019.

\bibitem[SAM{\etalchar{+}}20]{suzuki2018spectral}
Taiji Suzuki, Hiroshi Abe, Tomoya Murata, Shingo Horiuchi, Kotaro Ito, Tokuma
  Wachi, So~Hirai, Masatoshi Yukishima, and Tomoaki Nishimura.
\newblock Spectral pruning: Compressing deep neural networks via spectral
  analysis and its generalization error.
\newblock In {\em International Joint Conference on Artificial Intelligence},
  pages 2839--2846, 2020.

\bibitem[SAN20]{Suzuki2020Compression}
Taiji Suzuki, Hiroshi Abe, and Tomoaki Nishimura.
\newblock Compression based bound for non-compressed network: unified
  generalization error analysis of large compressible deep neural network.
\newblock In {\em International Conference on Learning Representations}, 2020.

\bibitem[Sch09]{schroeder2009fractals}
Manfred Schroeder.
\newblock {\em Fractals, chaos, power laws: Minutes from an infinite paradise}.
\newblock Courier Corporation, 2009.

\bibitem[Sch19]{schweinhart2019persistent}
Benjamin Schweinhart.
\newblock Persistent homology and the upper box dimension.
\newblock {\em Discrete \& Computational Geometry}, pages 1--34, 2019.

\bibitem[Sch20]{schweinhart2020fractal}
Benjamin Schweinhart.
\newblock Fractal dimension and the persistent homology of random geometric
  complexes.
\newblock {\em Advances in Mathematics}, 372:107291, 2020.

\bibitem[{\c{S}}GN{\etalchar{+}}19]{csimcsekli2019heavy}
Umut {\c{S}}im{\c{s}}ekli, Mert G{\"u}rb{\"u}zbalaban, Thanh~Huy Nguyen,
  Ga{\"e}l Richard, and Levent Sagun.
\newblock On the heavy-tailed theory of stochastic gradient descent for deep
  neural networks.
\newblock {\em arXiv preprint arXiv:1912.00018}, 2019.

\bibitem[SSDE20]{simsekli2020hausdorff}
Umut Simsekli, Ozan Sener, George Deligiannidis, and Murat~A Erdogdu.
\newblock Hausdorff dimension, heavy tails, and generalization in neural
  networks.
\newblock In {\em Proceedings of the Advances in Neural Information Processing
  Systems (NeurIPS)}, volume~33, 2020.

\bibitem[SSG19]{simsekli2019tail}
Umut Simsekli, Levent Sagun, and Mert Gurbuzbalaban.
\newblock A tail-index analysis of stochastic gradient noise in deep neural
  networks.
\newblock In {\em International Conference on Machine Learning}, pages
  5827--5837. PMLR, 2019.

\bibitem[ST94]{samorodnitsky1994stable}
G.~Samorodnitsky and M.~S. Taqqu.
\newblock {\em Stable non-{G}aussian random processes: stochastic models with
  infinite variance}, volume~1.
\newblock CRC press, 1994.

\bibitem[SZTG20]{simsekli2020fractional}
Umut Simsekli, Lingjiong Zhu, Yee~Whye Teh, and Mert Gurbuzbalaban.
\newblock Fractional underdamped langevin dynamics: Retargeting sgd with
  momentum under heavy-tailed gradient noise.
\newblock In {\em International Conference on Machine Learning}, pages
  8970--8980. PMLR, 2020.

\bibitem[TDSL00]{tenenbaum2000global}
Joshua~B Tenenbaum, Vin De~Silva, and John~C Langford.
\newblock A global geometric framework for nonlinear dimensionality reduction.
\newblock {\em science}, 290(5500):2319--2323, 2000.

\bibitem[TH12]{Tieleman2012}
T.~Tieleman and G.~Hinton.
\newblock {Lecture 6.5---RmsProp: Divide the gradient by a running average of
  its recent magnitude}.
\newblock COURSERA: Neural Networks for Machine Learning, 2012.

\bibitem[Vap68]{vapnik1968uniform}
Vladimir Vapnik.
\newblock On the uniform convergence of relative frequencies of events to their
  probabilities.
\newblock In {\em Doklady Akademii Nauk USSR}, volume 181, pages 781--787,
  1968.

\bibitem[ZBH{\etalchar{+}}21]{zhang2021understanding}
Chiyuan Zhang, Samy Bengio, Moritz Hardt, Benjamin Recht, and Oriol Vinyals.
\newblock Understanding deep learning (still) requires rethinking
  generalization.
\newblock {\em Communications of the ACM}, 64(3):107--115, 2021.

\bibitem[ZBL{\etalchar{+}}20]{zhao2020quaternion}
Yongheng Zhao, Tolga Birdal, Jan~Eric Lenssen, Emanuele Menegatti, Leonidas
  Guibas, and Federico Tombari.
\newblock Quaternion equivariant capsule networks for 3d point clouds.
\newblock In {\em European Conference on Computer Vision}, pages 1--19.
  Springer, 2020.

\bibitem[ZFM{\etalchar{+}}20]{zhou2020towards}
Pan Zhou, Jiashi Feng, Chao Ma, Caiming Xiong, Steven Hoi, and Weinan E.
\newblock Towards theoretically understanding why {SGD} generalizes better than
  {ADAM} in deep learning.
\newblock In {\em Advances in Neural Information Processing Systems (NeurIPS)},
  volume~33, 2020.

\bibitem[Zom10]{zomorodian2010fast}
Afra Zomorodian.
\newblock Fast construction of the vietoris-rips complex.
\newblock {\em Computers \& Graphics}, 34(3):263--271, 2010.

\bibitem[ZQH{\etalchar{+}}18]{zhu2018ldmnet}
Wei Zhu, Qiang Qiu, Jiaji Huang, Robert Calderbank, Guillermo Sapiro, and
  Ingrid Daubechies.
\newblock Ldmnet: Low dimensional manifold regularized neural networks.
\newblock In {\em Proceedings of the IEEE/CVF Conference on Computer Vision and
  Pattern Recognition (CVPR)}, pages 2743--2751, 2018.

\end{thebibliography}
\newcommand{\etalchar}[1]{$^{#1}$}

\newpage
\appendix
\section*{Appendices}

\section{Discussions}\label{sec:supp:discuss}

\paragraph{Connection to minimum spanning tree dimension. %
} 

In this section, we will describe another very related notion of dimension, called the \emph{minimum spanning tree} (MST) dimension. The MST dimension coincides with the PH dimension of order $0$, and provides further insights about the semantics of the PH dimension and more information about what the `distance-matrix' illustrations topologically represent in Figure~4 in the main document, and Figure~\ref{fig:cifar10vgg} in this document.

In order to define the MST dimension formally, let us introduce the required constructs. 
Let $W \subset \rset^d$ be a finite set of $K$ points, $W = \{w_1, \dots, w_K\}$. We consider a \emph{weighted fully connected graph} $\Graph$ generated from $W$, such that the vertices of the graph are the points in $W$, i.e., $\{w_1, \dots, w_K\}$, and the edge between the vertices $w_i$ and $w_j$ is set to the Euclidean distance between the vertices, i.e., $\|w_i -w_j\|$. Given this graph, we will consider spanning trees over $\Graph$, where a spanning tree over $\Graph$ is a tree whose nodes cover the vertices of $\Graph$ and the weights between the nodes are (still) the Euclidean distance between the nodes. 

In the rest of the section, with a slight abuse, we will use the notation \emph{``$T$ is a spanning over the set $W$''}, meaning that $T$ is a spanning tree over the graph $\Graph$ generated by the set $W$. The notation $e \in T$ will imply that $e$ is an edge in the tree $T$.

 \begin{dfn}[Total Edge Length]
     Let a $W \subset \rset^d$ be a {finite set} and $T$ be a spanning tree over $W$. Then, the total edge length of $T$ for ($0<\alpha<d$) is defined as:
     \begin{align}
         E_\alpha(T) = \sum\limits_{e\in T} |e|^\alpha = \sum\limits_{(w_i, w_j)\in T} \|w_i-w_j\|^\alpha. 
     \end{align}
 \end{dfn}

 \begin{dfn}[Minimum Spanning Tree (MST)] The minimum spanning tree (MST) is obtained as the minimizer $\mathrm{MST}(W) = \argmin_T E_1(T)$.
 \end{dfn}
 
 Now, we are ready to define the MST dimension.
  \begin{dfn}[MST Dimension \cite{kozma2006minimal}]
For a {bounded metric space} $\Wcal$, the MST dimension is defined as the infimal exponent $\alpha$ s.t. $E_\alpha(\mathrm{MST}(W))$ is uniformly bounded for \textbf{all finite point sets} $W \subset \Wcal$. Formally,
\begin{equation}
\dimMST \Wcal := \inf \big\{ \alpha \,:\,\exists C \in \rset_+; \>\>  E_\alpha(\mathrm{MST}(W)) \leq C, \quad \forall W \subset \Wcal, \> \mathrm{card}(W) < \infty \big\}
\end{equation}
where $\mathrm{card}(W)$ denotes the cardinality of $W$.
\end{dfn}

By the seminal work of Kozma \etal \cite{kozma2006minimal} (also \cite{schweinhart2019persistent}), we have that, for all bounded $\Wcal \subset \rset^d$:
\begin{align}
   \dimMST \Wcal = \dimPH^0 \Wcal = \dimBox \Wcal.
\end{align}
This relation provides further intuition over what topological properties of $\Wcal$ are taken into account by the MST dimension, hence the PH dimension. We observe that, given a finite set of points $W\subset \Wcal$, the MST dimension summarizes the topology of $\Wcal$ by the graph generated by $W$, which is essentially the distance matrix given in Figure~4 in the main paper, and Figures~\ref{fig:cifar10vgg} and \ref{fig:vggmnist} here. Then, the growth rate of the total edge lengths of the MSTs over $W \subset \Wcal$ determine the intrinsic dimension of $\Wcal$. This relationship provides a more formal illustration to the visualization provided in the experiments section: the structure of the distance matrices (e.g., the clustering behavior) has indeed a crucial role in the intrinsic dimension, hence in generalization performance by \cite{simsekli2020hausdorff}. Therefore, a fine inspection of the distance matrix obtained through the SGD iterations can be predictive of the generalization performance.

\paragraph{On the intrinsic dimension estimators.} 
The notion of \emph{intrinsic dimension} (ID) is introduced by Robert Bennett~\cite{bennett1969intrinsic} in the context of signal processing as the minimum number of parameters needed to generate a data description so that (i) information loss is minimized, (ii) the `intrinsic' structure characterizing the dataset is preserved. 
Following its definition, different ID estimators were proposed with different design goals such as computational feasibility, robustness to high dimensionality \& data scaling or accuracy~\cite{camastra2016intrinsic}. Most of these estimators either are strictly related to Hausdorff dimension (HD)~\cite{hausdorff1918dimension} or try to reconstruct the underlying smooth manifold, leading to the broader notion of \emph{manifold dimension} or \emph{local intrinsic dimension}. On one hand, the former, HD, is very hard to estimate in practice. On the other hand, explicitly reconstructing the manifold or its geodesics as in ISOMAP~\cite{tenenbaum2000global} can be computationally prohibitive. 

These challenges fostered the development of a literature populated with measures of various intrinsic aspects of data or its samples~\cite{facco2017estimating, campadelli2015intrinsic,schroeder2009fractals}. Geometric approaches like \emph{fractal} measures usually replace the Hausdorff dimension with a lower or upper bound aiming for scalable algorithms with reasonable complexity.
Kegl~\etal~\cite{kegl2002intrinsic} try to directly estimate an upper bound, the Box Counting or Minkowski dimension. Though, their Packing Numbers are not multiscale and have quadratic complexity in the cardinality of the data. 
Correlation dimension (\emph{strange attractor dimension} in chaotic dynamical systems) on the other hand is a lower bound on the Box Counting dimension~\cite{grassberger2004measuring} and is efficient to compute. Another fractal dimension estimator, Maximum Likelihood Estimation (MLE) dimension~\cite{levina2005maximum} estimates the expectation value of Correlation dimension. Finally, Information dimension~\cite{isham1993statistical} measures the fractal dimension of a probability distribution.
As we show in Thm. 1 of the main paper, our PH dimension measures the upper bound, Minkowski Dimension exactly and is related to the generalization in neural networks.

The other line of work trying to estimate a \emph{manifold dimension} makes the assumption that the dataset or in our case the weights of neural networks lie on a smooth and compact manifold. The \emph{geodesic minimum spanning tree length} of Costa and Hero try to approximate the geodesics of this supposedly continuous manifold~\cite{costa2003manifold}. Lin and Zha~\cite{lin2008riemannian} reconstruct the underlying \emph{Riemannian Manifold} with its local charts through 1D simplices ($k$-nearest edges). The dimension is then estimated by performing a local PCA~\cite{jolliffe1986principal}. It is interesting and noteworthy that such parallel track of works based upon a discretization of the underlying smooth manifold can yield algorithms very similar to ours. Though, we are not aware of well established connections between these fractal based and manifold based approaches. Evaluating the performance of these manifold dimension methods is out of the scope of this paper since unlike our method, they are not explicitly linked to generalization.

We are, however, curious to see how the estimators related to Box Counting dimension perform in our setting. To do that, we use a similar synthetic diffusion process as we did in the main paper. This allows us to have access to the ground truth dimension as the tail index of a $\beta$-stable process (see below for the definition). In~\cref{fig:dimest} we put correlation dimension (Corr), MLE dimension (MLE) and our dimension ($\PHrobust$) at test. For sanity check, we also include an \emph{eigen-value} based ID estimator \emph{PCA Dimension}~\cite{jolliffe1986principal}. It is seen that all of these methods over-estimate the dimension in small-data regime and approach the correct dimension as $n$ grows. As expected, our method -- explicitly connected to the tail index -- outperforms the other methods. Correlation dimension performs very reasonably in this experiment, only slightly worse. It is also noteworthy that the PCA Dimension estimates are not well correlated with GT tail index (they do not even fit the axis limits of our plot) and therefore cannot be used in our framework to measure generalization in deep networks. This is because (i) they are expected to fail on nonlinear manifolds such as diffusion or network trajectories; (ii) they do not measure per se the fractal dimension, which are useful in evaluating the model order in time series data or nonlinear dynamical systems.

\insertimageC{1}{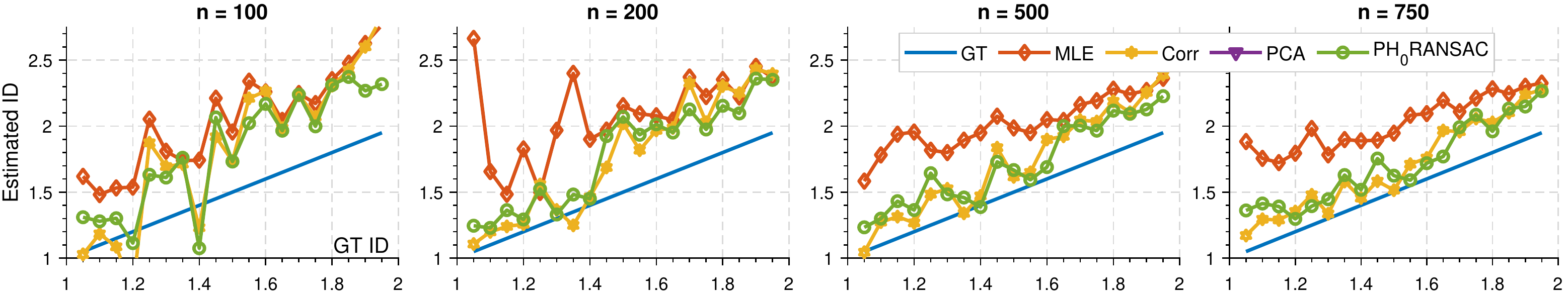}{Comparing different dimensions estimators for different number of sample points $n$. We simulate the point cloud data as the trajectories of a diffusion using $\beta$-stable processes in $d=128$ ambient dimensions. The tail index $\beta$ mapped to the $x$-axis of all plots corresponds to the ground truth (GT) intrinsic dimension. $\beta\rightarrow 1$ models a heavy-tailed process whereas $\beta\rightarrow 2$ is a Brownian motion. Our method and correlation dimension can capture the intrinsic properties of this data. Yet, our method performs slightly better as it is a theoretically grounded measure of the tail index.}{fig:dimest}{t!}

Among different dimension estimators, the Two-NN estimator~\cite{facco2017estimating} has recently been found to be of practical value in the context of measuring the intrinsic dimension of data representations, or the layers of the network~\cite{ansuini2019intrinsic}. However, this estimator is based upon the assumption of approximately constant density on the length-scale defined by the typical distance to the second neighbor. In the main paper we showed that this assumption breaks for the processes that exhibit heavy tails, such as the trajectories of deep network training algorithms, making Two-NN a bad estimator. In~\cref{fig:sphereplots}, we show that for data where this assumption holds -- such as randomly sampled points on a $d$-dimensional hypersphere, Two-NN can be a good estimator. Yet, our estimator works both in this case and in the case of heavy tailed diffusion-like processes.

\paragraph{On the interpretation of Figure 4 -- middle panel.}

We shall note two important points regarding the plots given in Figure 4, middle panel:

\begin{enumerate}
    \item The slope of the fitted line essentially determines the intrinsic dimension (see the paragraph after Proposition 2 to see the mathematical relation) and we illustrate this behavior visually in the plots: the slope of the data computed by using persistent homology determines the intrinsic dimension and hence determines the generalization gap; which is arguably a surprising result.
    \item Even though we fit a line to the empirical data, it does not automatically mean that the empirical data should have a clear linear trend. These plots further demonstrate that we fit a line to a data, which has a very strong linear trend, hence our estimations are not jeopardized by noise or model mismatch.
\end{enumerate}

\paragraph{On the interpretation of distance matrices.}
The columns and rows of the distance matrices shown here and in the main paper are organized with respect to the iteration indices after convergence: we first train the networks until convergence, and then run an additional 200 iterations near the local minimum. Then, the $(i,j)$-th entry of a distance matrix corresponds to the Euclidean distance between the $i$-th iterate and the $j$-th iterate (of these additional 200 iterations).

Though we do not yet have a rigorous proof, we believe that the qualitative difference in these diagrams is due to the heavy-tailed behavior of the SGD algorithm \cite{simsekli2019tail,gurbuzbalaban2021heavy,hodgkinson2020multiplicative}. Let us illustrate this point with a simpler example: consider the Levy $\beta$-stable process used in Section 5.2 of the main paper. This is a well-known heavy-tailed process which becomes heavier-tailed when the parameter $\beta$ decreases. A classical result in probability theory~\cite{blumenthal1960some} shows that the Hausdorff dimension of the trajectories of this process is (almost surely) equal to its tail-index $\beta$ for any $d\geq 2$. This means that as the process gets heavier-tailed, its intrinsic dimension decreases.

On the other hand, if we investigate the geometric properties of heavy-tailed processes, we mainly observe the following behavior: the process behaves like a diffusion for some amount of time, then it makes a big jump, then it repeats this procedure. For a visual illustration, we can recommend Figure 1 in~\cite{simsekli2020hausdorff}. Finally, this \textit{diffuse+jump} structure creates clusters, where each cluster is created during the \textit{diffuse} period and a new cluster is initiated at every large \textit{jump}.

Now, coming back to the question of interpreting the distance matrices, the non-uniform pixelations in the distance matrices are well-known indications of clusters in data: in the top row of Figure 4, the large dark squares indicate different clusters, and we observe that these clusters become more prominent for smaller PH dimensions. To sum up:
\begin{enumerate}[label=(\roman*),itemsep=0pt,topsep=0pt,leftmargin=*,align=left]
    \item SGD can show heavy-tailed behavior when the learning-rate/batch-size is chosen appropriately~\cite{simsekli2020hausdorff,gurbuzbalaban2021heavy}.
    \item Heavy-tails result in a topology with smaller dimension~\cite{simsekli2020hausdorff}, and creates a clustering behavior in the trajectories.
    \item  The clustering behavior results in non-uniform pixelations in the distance matrices and gaps in the persistence diagrams.
\end{enumerate}

\insertimageC{1}{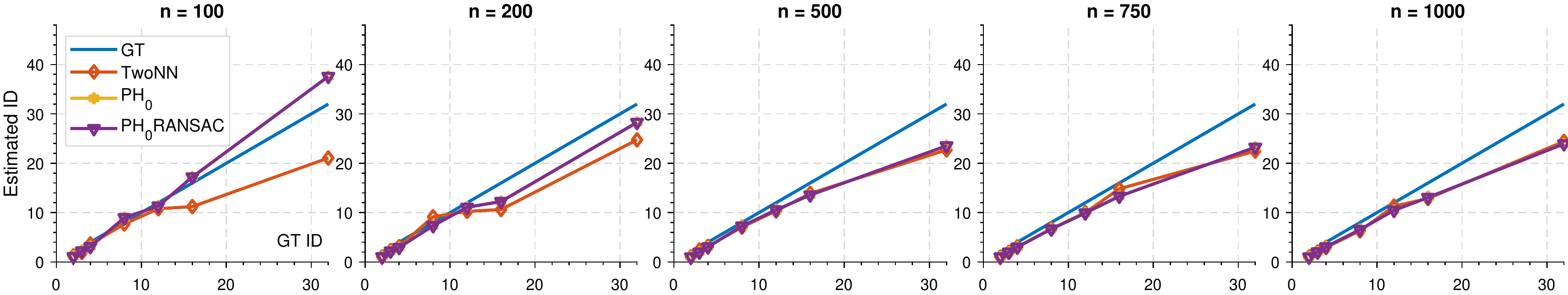}{Results of different dimension estimators for synthetic sphere data. $\mathrm{GT}$, $\mathrm{TwoNN}$, $\PH_0$ and $\PH_0\mathrm{RANSAC}$ denote the ground truth, TwoNN estimator~\cite{facco2017estimating}, our standard persistent homology estimator and its robust line fitting variant, respectively.}{fig:sphereplots}{t}
\section{Details of Experimental Evaluation}
\label{supp:exp_details}
\paragraph{Analysis} For our analysis, we train with the following architectures. For basic networks, we include fully connected models with $5$ (fcn-5) and $7$ (fcn-7) layers and a $9$-layer convolutional network (cnn-9). All networks have ReLu activation. We also include several more standard networks such as AlexNet \cite{krizhevsky2012imagenet}.
We train using variants of SGD without momentum or weight-decay. We trained the networks with different step-sizes in the range $[0.001, 0.1]$ and batch-sizes in the set $\{64,100, 128\}$. We trained all models until convergence.
We implemented our code in PyTorch and conducted the experiments on $4$ GeForce GTX 1080 GPUs. We used the classical CIFAR10 and MNIST datasets and for measuring the training and test accuracies, we use the standard training-test splits.

\paragraph{The $\beta$-stable Levy process} In our experiments, we illustrated our approach on estimating the intrinsic dimension of $\beta$-stable Levy processes\footnote{This processes are often called the $\alpha$-stable processes, where $\alpha$ denotes the stability index of the process \cite{samorodnitsky1994stable}. However, to avoid confusion with the parameter $\alpha$ in $E_\alpha$, we denote the stability index with the symbol $\beta$.}, whose definition is given as follows.

For $\beta \in (0,2]$, a $\beta$-stable L\'{e}vy process $\{\mathrm{L}^\beta_t\}_{t\geq 0}$ in $\rset^d$ with the initial point $\mathrm{L}^\beta_0 = 0$,
is defined by the following properties:
\begin{enumerate}[label=(\roman*),itemsep=0pt,topsep=0pt,leftmargin=*,align=left]
\item For $N\in \mathbb{N}$ and $t_0<t_1 < \cdots < t_N$, the increments $ (\mathrm{L}^\beta_{t_{i}} - \mathrm{L}^\beta_{t_{i-1}} )$ are independent for all $i$.%
\item For any $t>s>0$,  $(\mathrm{L}^\beta_t - \mathrm{L}^\beta_s)$ and $\mathrm{L}^\beta_{t-s}$ have the same distribution, and $\mathrm{L}^\beta_1$ has the characteristic function $\exp(- \|\omega\|^\beta)$. %
\item $\mathrm{L}^\beta_t$ is continuous in probability, i.e., for all $\delta >0$ and $s\geq 0$, $\mathbb{P}(|\mathrm{L}^\beta_t - \mathrm{L}^\beta_s| > \delta) \rightarrow 0$ as $t \rightarrow s$.
\end{enumerate}

The stability index (also called the tail-index) $\beta$ determines the tail-behavior of the process (e.g., the process is heavy-tailed whenever $\beta <2$) \cite{samorodnitsky1994stable}. On the other hand, interestingly $\beta$ further determines the geometry of the trajectories of the process, in the sense that, \cite{blumenthal1960some} showed that the Hausdorff dimension of the trajectory $\{\mathrm{L}^\beta_t\}_{t\in [0, T]}$ is equal to the tail-index $\beta$ almost surely. Similar results have also been shown for the box dimension of the trajectories as well \cite{falconer2004fractal}. 
Thanks to their flexibility, $\beta$-stable processes have been used in the recent deep learning theory literature \cite{simsekli2019tail,zhou2020towards,simsekli2020hausdorff}.

\section{Additional Evaluations}
\label{supp:exp}

\begin{wrapfigure}[12]{r}{0.495\textwidth} %
\vspace{-4mm}
\includegraphics[width=0.495\textwidth]{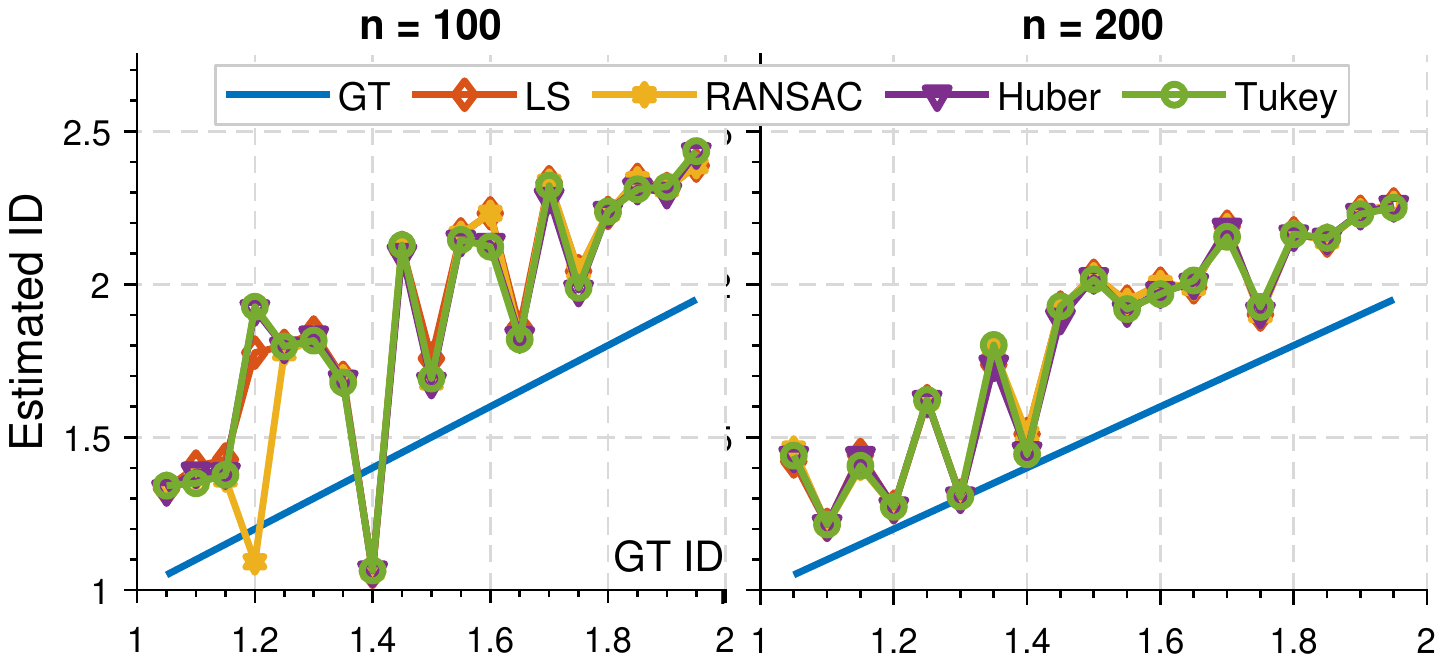}
\vspace{-3mm}
\caption{Effect of different robust functions on dimension estimation.}
    \label{fig:robust}
\end{wrapfigure}

\paragraph{Use of different robust functions.}
Our dimension estimate is essentially the slope of a line fit to the log-sequence of total edge lengths. Since a slight change in the slope can directly influence the resulting dimension, it is of interest to see the effect of different line estimators. To this end, we evaluate the naive least squares (LS) fit as well as a series of methods that are robust to the outliers. In particular, we deploy random sample consensus (RANSAC)~\cite{fischler1981random} as well as Huber~\cite{huber1992robust}, Tukey~\cite{beaton1974fitting} robust M-estimators. As seen in~\cref{fig:robust}, the difference between them in dimension estimation is negligible. This indicates that, as expected, the data points composed of total edge lengths are not corrupted by outliers -- every iterate (network) in the trajectory yields a valid data point. We used the RANSAC estimator as it gave a marginally better estimate.

\paragraph{Effect of regularization on the intrinsic dimension.} For our regularization procedure, we train the LeNet \cite{LeCun1998GradientbasedLA} architecture on Cifar10 with a batch size of 128. We again optimize with a constant step size SGD varying the learning rate from $0.01$ to $0.1$. The exact training procedure is summarized in~\cref{algo:reg} where we show how to apply our dimension constraint in a sliding fashion.

\begin{algorithm2e} [t!]
\DontPrintSemicolon
\SetKwInOut{Input}{input}
\SetKwInOut{Output}{output}
\Input{The set of data $\mathcal{D}=\{(x_i, y_i)\}_{i=1}^D$, gradient update method $\mathrm{gd}$, number of networks for dimension computation $K$, regularization constant $\lambda$, number of steps $N$, batch size $b$, regular training loss $\mathcal{L}$}
\Output{Regularized trained network $W$}
$n \gets 0$, initial network $W_0$, network list $Q$.
\While{$n\leq N$}{
$\{x_i, y_i\}_{\mathcal{S} \subset [D]} \gets \mathrm{sample}(W, b)$\,{\color{purple} \small \tcp{random sampling}}
$Q \gets Q \cup \{W_n\}$\,{\color{purple} \small \tcp{append the last network}}
$\mathcal{L}_{top} \gets \mathcal{L}(x_i, y_i) + \lambda \cdot  \mathrm{compute\_dims}(Q, K)$\, {\color{purple} \small \tcp{Compute loss of last $K$ models}} 
$W_{n + 1} = \mathrm{gd}(W_n, \nabla_{W_n} \mathcal{L}_{top}), n \gets n + 1$\,{\color{purple} \small \tcp{update weights w.r.t. gradients}}
}
Return $W_N$
\caption{Topological Regularization Training}
\label{algo:reg}
\end{algorithm2e}

Here, we test to see if our regularization would be able to control the intrinsic dimension of the training trajectories. In particular, we calculate the intrinsic dimension of our regularization experiments and report them in~\cref{fig:reg_dim}. It is noticeable that our topological regularizer is indeed able to decrease the intrinsic dimension for different training regimes with different learning rates. This property is well reflected to the advantage in Fig. 5 of the main paper.
\begin{figure}[ht]
    \centering
    \begin{tabular}{c c c c c c}
        Learning rate & 0.001 & 0.0025 & 0.005 & 0.0075 & 0.01 \\ 
        \hline Unregularized & 3.21 & 3.35 & 3.54 & 3.67 & 4.2 \\ 
        Regularized & 3.13 & 3.25 & 3.30 & 3.43 & 3.75
    \end{tabular}
    \caption{The intrinsic dimensions (ID) of unregularized vs topologically regularized network training. Across the different learning rates, our regularized network has a better behaved (lower) ID.}
    \label{fig:reg_dim}
\end{figure}
\insertimageC{1}{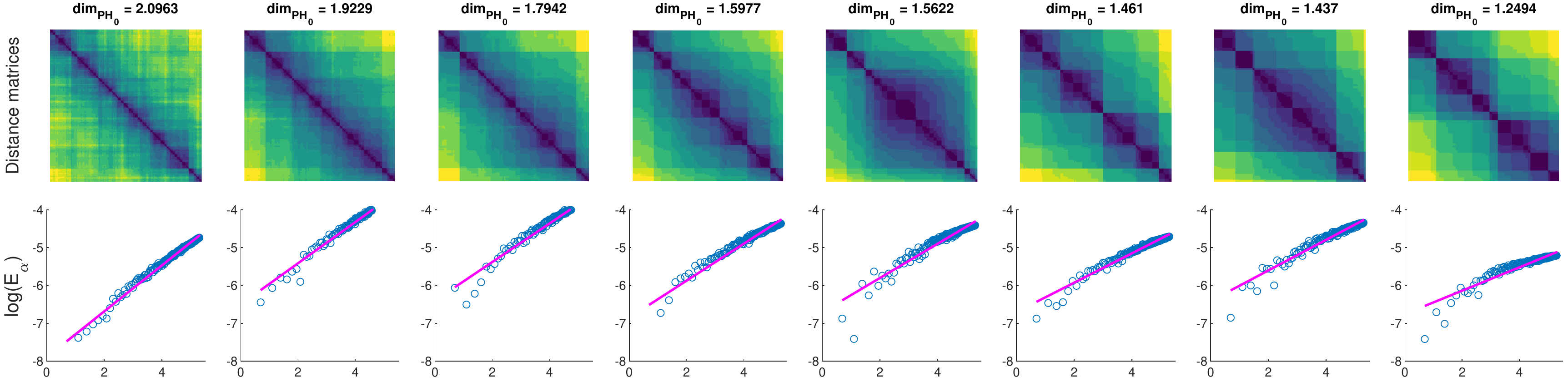}{Visualization for cnn-9 network on MNIST dataset sorted by persistent homology dimension. (\textbf{top}) the distance matrices computed between the network weights corresponding to the last $200$ iterates. (\textbf{bottom}): the behavior of the logarithmic $\alpha$-weighted lifetime sums $\log E_\alpha^0$, derived from the persistent homology with respect to $\log(n)$.}{fig:vggmnist}{t}
\insertimageC{1}{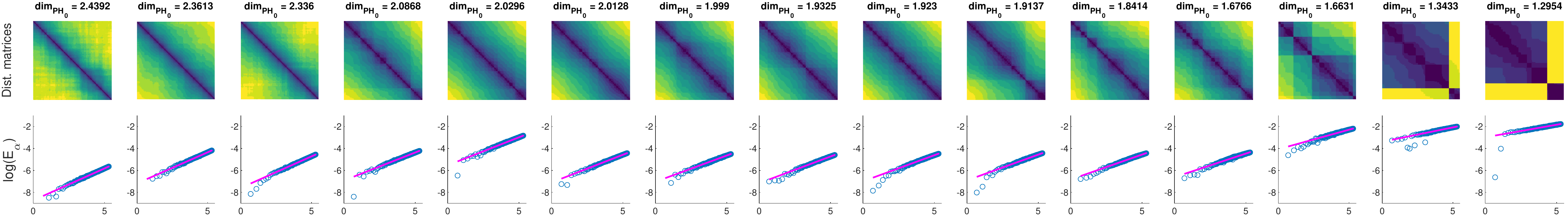}{Visualization for VGG-like cnn-9 network on Cifar10 dataset sorted by PH dimension. Please see~\cref{fig:vggmnist} for more information on the plots.}{fig:cifar10vgg}{t}
\insertimageC{1}{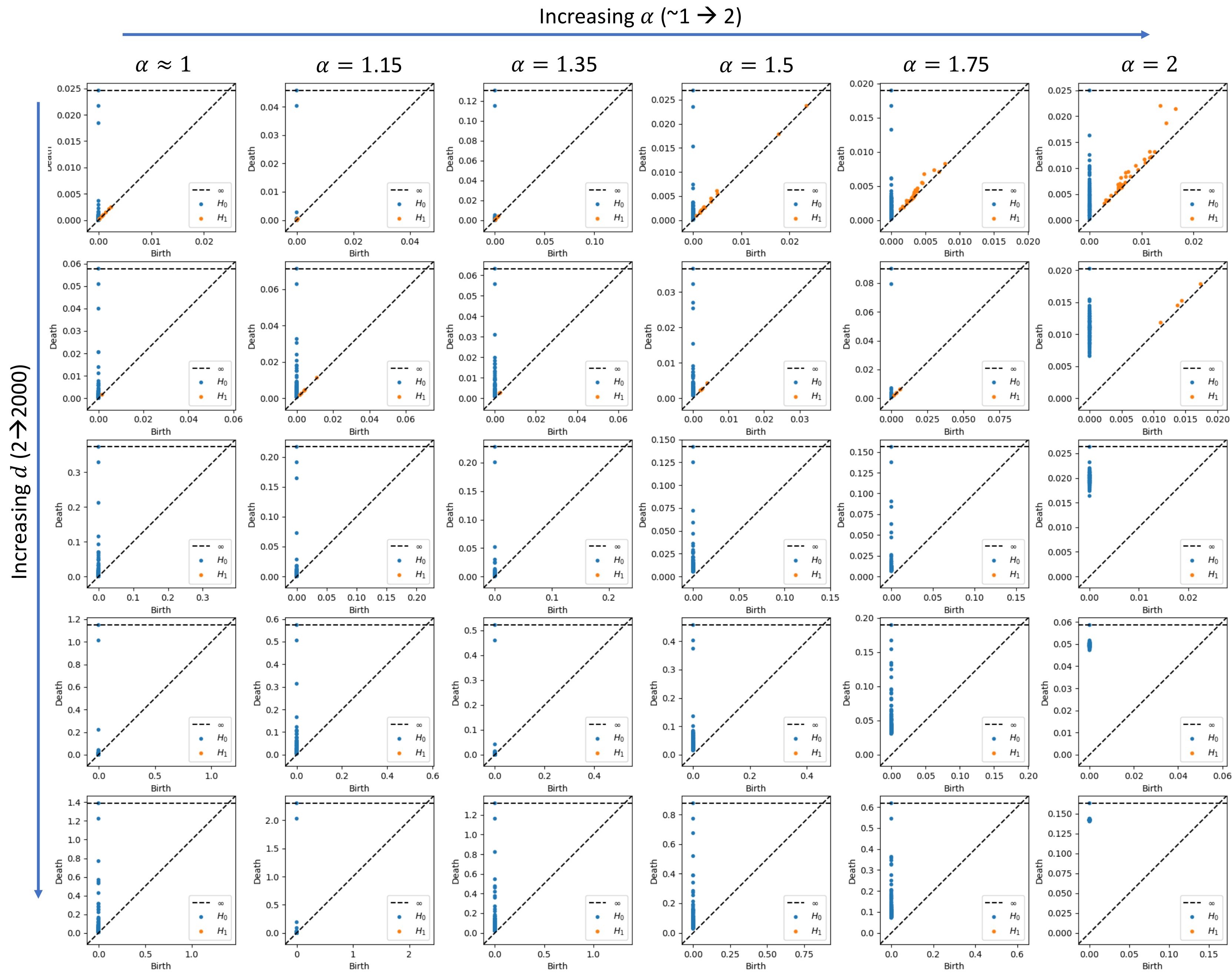}{Persistence Diagrams for the diffusion processes of varying ambient dimension and intrinsic dimension (tail index).}{fig:PD}{t!}
\paragraph{Further visualizations on cnn-9 network.}
In the main paper we have visualized the distance matrices, total edge lengths $E_\alpha^0$, and the persistence diagrams for the 7-fcn network. We have observed a clustering effect with decreasing persistent homology dimension. One could question whether the same effect is observable in other networks and for other datasets. Hence, we now visualize the same quantities for our cnn-9 (VGG-like) network both on Cifar10 and MNIST datasets. Figures~\ref{fig:vggmnist} and~\ref{fig:cifar10vgg} plot the distance matrices and $\log E_\alpha^0$-values in a fashion identical to the main paper for MNIST and Cifar10 datasets, respectively. Note that in both figures there is a clear \emph{pixelation} effect on the distance matrices as the PH-dimension decreases. This indicates a clustering as expected from a diffusion process. We measures this by the slope of the line as shown in the bottom rows. We conclude that the behavior reported in the main paper is consistent across datasets and networks.

\paragraph{Further visualizations on synthetic data.}
Last but not least, we show the effect of the change in intrinsic dimension to the persistence diagram. To this end, we synthesize point clouds of different ambient and intrinsic dimensions from the aforementioned diffusion process. We then compute the persistence diagram of all the point clouds and arrange them into a grid as shown in ~\cref{fig:PD}. This reveals two observations: (i) As $\alpha\rightarrow 1$, the persistence diagram gets closer (in the relative sense of the word) to the diagonal line while forming distinct clusters of points. Note that as death is always larger than birth, by definition, these plots can only be upper-triangular. As we approach $\alpha\rightarrow 2$, the process starts resembling a Brownian motion and we observe a rather dispersed distribution while the points move away from the diagonal line (relatively), at least for the case we are interested in, $H_0$.

\end{document}